\documentclass[12pt]{article}

\newcommand{\blind}{0}

\usepackage{amsmath,amsthm,amssymb,bbm}
\usepackage{setspace, fancybox, fullpage}
\usepackage{chngcntr}
\usepackage[dvips]{graphicx,color}
\usepackage{graphicx,xcolor}
\usepackage[export]{adjustbox}
\usepackage{float}
\usepackage{mathrsfs}
\usepackage{framed}
\usepackage{listings}
\usepackage{epstopdf}
\usepackage{wrapfig,color,bm}
\usepackage{epsfig}
\usepackage{multirow}
\usepackage{enumerate}
\usepackage{setspace}
\usepackage{exscale}
\usepackage{relsize}
\usepackage{array}
\usepackage{makecell}
\usepackage{geometry}
\usepackage{booktabs}
\usepackage{varwidth}
\usepackage{longtable, booktabs}
\usepackage{supertabular}
\usepackage{algorithm}
\usepackage{algorithmicx}
\usepackage[noend]{algpseudocode}
\usepackage[round]{natbib}
\usepackage[colorlinks,citecolor=blue,linkcolor=blue,hypertexnames=true]{hyperref}
\usepackage[capitalize,noabbrev]{cleveref}

\newtheorem{lemma}{Lemma}
\newtheorem{theorem}{Theorem}

\newtheorem{proposition}[theorem]{Proposition}
\newtheorem{assumption}{Assumption}

\makeatletter
\algnewcommand{\parState}[1]{\State%
  \parbox[t]{\dimexpr\linewidth-\the\ALG@thistlm}{\strut #1\strut}}

\newcolumntype{H}{>{\setbox0=\hbox\bgroup}c<{\egroup}@{}}

\def\argmin{\mathop{\rm argmin}}
\def\argmax{\mathop{\rm argmax}}

\def\arginf{\mathop{\rm arginf}}

\DeclareMathOperator{\sign}{sign}

\graphicspath{{figures/}}

\date{}

\begin{document}

\def\spacingset#1{\renewcommand{\baselinestretch}%
{#1}\small\normalsize} \spacingset{1}

\if0\blind
{
  \title{\bf Learning Acceptance Regions for Many Classes with Anomaly Detection}
  \author{Zhou Wang and Xingye Qiao\thanks{Correspondence to: Xingye Qiao (e-mail: qiao@math.binghamton.edu). Zhou Wang is a PhD student and Xingye Qiao is a professor in the Department of Mathematics and Statistics at Binghamton University, State University of New York, Binghamton, New York, 13902-6000.}}
  \maketitle
} \fi

\if1\blind
{
  \bigskip
  \bigskip
  \bigskip
  \begin{center}
    {\LARGE\bf Learning Acceptance Regions for Many Classes with Anomaly Detection}
  \end{center}
  \medskip
} \fi

\bigskip
\begin{abstract}
\noindent Set-valued classification, a new classification paradigm that aims to identify all the plausible classes that an observation belongs to, can be obtained by learning the acceptance regions for all classes. Many existing set-valued classification methods do not consider the possibility that a new class that never appeared in the training data appears in the test data. Moreover, they are computationally expensive when the number of classes is large. We propose a Generalized Prediction Set (GPS) approach to estimate the acceptance regions while considering the possibility of a new class in the test data. The proposed classifier minimizes the expected size of the prediction set while guaranteeing that the class-specific accuracy is at least a pre-specified value. Unlike previous methods, the proposed method achieves a good balance between accuracy, efficiency, and anomaly detection rate. Moreover, our method can be applied in parallel to all the classes to alleviate the computational burden. Both theoretical analysis and numerical experiments are conducted to illustrate the effectiveness of the proposed method.
\end{abstract}

\noindent%
\vfill

\newpage
\spacingset{1.4} 

\setcounter{page}{1}
\abovedisplayskip=8pt
\belowdisplayskip=8pt

\section{Introduction}
\label{introduction}
In multicategory classification, traditional methods return a single class label as the prediction without a confidence measure attached. For points near the classification boundary where the classes overlap, these methods may misclassify with high probability. As classification and machine learning in general have played a more and more significant role in high stake domains, these mistakes can incur severe consequences. To avoid making mistakes when they are likely to happen, set-valued classification methods have emerged
\citep{herbei2006classification, shafer2008tutorial, D_mbgen_2008, denis2017confidence, wang2018learning, zhang2018reject, sadinle2019least}. A set-valued classifier may return multiple class labels as the prediction for each observation. Specifically, those near the boundary between classes may receive multiple labels as the prediction.

\citet{herbei2006classification}, \citet{bartlett2008classification}, \citet{ramaswamy2015consistent} and \citet{zhang2018reject} proposed and developed Classification with a Reject Option (CRO) by training a classifier and a rejector at the same time. A rejector determines when to refuse to make a classification for ambiguous points (i.e. ambiguity rejection). In CRO, observations that are rejected have been predicted to a subset of all class labels. Conformal prediction \citep{vovk2005algorithmic, shafer2008tutorial, balasubramanian2014conformal} is another increasingly popular framework that outputs a prediction set with a pre-specified confidence guarantee. \citet{lei2014classification}, \citet{wang2018learning} and \citet{sadinle2019least} considered the set-valued classification from an optimization perspective. In particular, the goal is to minimize ambiguity (defined as the expected size of the prediction set) while controlling class-specific misclassification rates. \citet{denis2015consistency, denis2017confidence} worked with its dual problem, minimizing misclassification rates with the ambiguity controlled.

In many practical fields like intrusion detection, bank fraud prevention, and public health, new classes that did not exist in historical training data may appear as time goes by. The aforementioned set-valued classification methods would be forced to classify new-class observations to an existing class. It is therefore important to design classifiers that are capable of \textit{anomaly detection} (aka, outlier detection). In public health, decision-makers need to confidently  identify the strain for a prevailing virus in a community. The fewer candidates there are, the more effective preventive measures can be deployed. Moreover, the possibility of a new virus (e.g. a new COVID-19 virus variant) requires the detection of new strains. This example shows the necessity of a set-valued classifier with the capacity of anomaly detection.

Recent works on adapting set-valued classifiers to handle anomaly data often focus on using the conformal prediction framework: first, a score function is obtained; second, a cutoff is determined using conformal splits; third, new observations are classified by thresholding the score with the cutoff. \citet{hechtlinger2018cautious} utilized the covariates' density given class, $p(\bm x\mid y)$, as the score. However, the acceptance region for each class is learned with no regards to any other class; in the sense of minimizing the ambiguity, this approach was shown to be suboptimal \citep{D_mbgen_2008}.  \citet{guan2019prediction} considered the classification problem between a given class $k$ and the entire test data and thresholded the resulting estimates of $p(y\mid \bm x)$. Both methods depend on probability or density estimation, which is known to be a challenging task when the dimension $p$ is large \citep{wu2010robust, zhang2013effect}. It is hence desirable to propose a method without estimating probability. Moreover, the score functions in the aforementioned works were estimated without the goal of ultimately minimizing the ambiguity in mind. For example, though the true score $p(y\mid \bm x)$ can guarantee the minimization of the ambiguity, empirically a finite-sample estimate may not share this property. In this article, we propose  the Generalized Prediction Set (GPS) method to overcome these difficulties.

We have made four contributions in this article. First, we propose a new large-margin classification method for outlier detection without involving probability estimation. Our model is estimated by minimizing the empirical ambiguity and a penalty term that encourages outlier detection, subject to a bounded misclassification rate for each class.
Second, using weighted kernel and regularization, we enable feature selection for our method in high-dimensional settings.
Third, in contrast to methods that solve an optimization problem involving all the classes simultaneously, our proposed method is well positioned for parallel computing, hence, allowing fast classification even when there are many classes.
Finally, we conduct a thorough theoretical analysis of the proposed method, showing that its true misclassification rate is bounded, the excess ambiguity is bounded, and it has variable selection consistency.

The rest of the paper is organized as follows. Section \ref{sec:review} provides some background on anomaly detection and set-valued classification. We introduce our method and the implementation algorithms in Section \ref{sec:method}. Theoretical guarantees are provided in Section \ref{sec:theory}. In Section \ref{sec:numerical}, we compare the proposed method with competing methods using simulated and real data. Some concluding remarks are given in Section \ref{sec:conclusion} and proofs are in \cref{sec:proof}.


\section{Preliminaries}\label{sec:review}
In this section, we review the background of the anomaly detection problem and set-valued classification methods.
\subsection{Anomaly detection}\label{sec:anomalydet}

We define anomaly detection as the identification of new observations that do not belong to the same distributions as the existing observations. We use the terms anomaly detection, outlier detection, and novelty detection interchangeably. Commonly used anomaly detection methods include one-class SVM (OCSVM), deep one-class classification \citep{ruff2018deep}, density level set estimation \citep{breunig2000lof, chen2017density}, and positive-unlabeled learning (PU learning) \citep{du2014analysis}.

Suppose we have a random sample $\{\bm x_i\}_{i=1}^m$ from $\mathcal{X}$. Let $\Phi:\mathcal{X} \to \mathcal{H}$ be a kernel map from the input space to the feature space. OCSVM \citep{scholkopf2000support} aims to separate data features from the origin with a maximum margin. The OCSVM marks an observation $\bm x$ as an outlier if the decision function $f(\bm x):=\bm w^\top\Phi(\bm x)-\rho$ yields $f(\bm x)<0$, where $\bm w$ and $\rho$ are obtained by solving \begin{equation}
\min\limits_{\bm w, \rho}~\frac{\|\bm w\|^2}{2}+\frac{1}{m\nu}\sum\limits_{i=1}^m\xi_i-\rho,
~~~\mbox{s.t.}~ \bm w^\top\Phi(\bm x_i)\geq \rho-\xi_i, \ \xi_i\geq 0, \ i\in[m].
\end{equation}
Here the tuning parameter $\nu$ controls the number of observations treated as anomalies. To work with complex data, \citet{ruff2018deep, ruff2021unifying} achieved anomaly detection using deep learning, where new features are learned through a network and are then applied to the deep Support Vector Data Description (SVDD).

\citet{steinwart2005classification} showed that anomaly detection can be achieved by solving a classification problem between all the normal, existing classes combined and the abnormal class, assuming that the proportion of the abnormal class is known. Motivated by this equivalency, they proposed to train a cost-sensitive SVM between these two classes. To the same token, by having prior information about the anomaly class, \citet{liu2018open} proposed Open Category Detection (OCD) with theoretical guarantee to achieve a pre-specified outlier detection rate via estimating the corresponding distribution. However, in practice, abnormal class data are not observed. To mine information about the abnormal class, \citet{du2014analysis, du2015convex} proposed to classify between the entire training data with the entire test data; while the former contains normal classes only, the latter may contain abnormal classes. Various methods to estimate the proportion of the abnormal class have been proposed by \citet{elkan2008learning, blanchard2010semi} and \citet{christoffel2016class}.

Outlier detection can also be done in conjunction with a standard classification task. For example, both \citet{jumutc2013supervised} and \citet{hanczar2014combination} conducted two separated one-class SVMs in order to achieve binary classification and anomaly detection. \citet{hechtlinger2018cautious} proposed to use density level sets for the purpose of classification and anomaly detection. However, there is no theoretical guarantee for these methods. Other works related to classification and anomaly detection are open-set recognition (OSR) \citep{bendale2015towards} and out-of-distribution (OOD) detection \citep{yang2021oodsurvey}. Both of them are to reject observations with lower scores (potential anomalies) at first, and then conduct standard single-valued classification for non-rejected observations.

\subsection{Set-valued classification}\label{setvalue}
Consider a multicategory classification setting with input space $\mathcal{X}=\mathbb{R}^p$ and labels $\mathcal{Y}=\{1, \cdots, K\}$. Let $\left(\bm X, Y\right)\in \mathcal{X}\times \mathcal{Y}$ come from an unknown distribution $\mathcal{P}$. One way to obtain set-valued classifiers is to conduct a series of hypothesis tests that determine if the test observation belongs to a given class $k$. The set of all observations that are not rejected as being from class $k$ is called the acceptance region for class $k$, denoted as $\mathcal{C}_k\subset\mathcal{X}$. Given $\mathcal{C}_k$, $k\in[K],$ a set-valued classifier $\phi:\mathcal{X}\mapsto 2^\mathcal{Y}$ can be defined as $\phi(\bm x) =\{k: \bm x\in \mathcal{C}_k\}$, that is, all the classes whose acceptance regions contain $\bm x$. Typically, there are two competing metrics for a set-valued classifier, accuracy, and efficiency. The accuracy may be quantified by misclassification rate $\mathbb{P}\left(Y\not\in \phi\left(\bm X\right)\right)$ or class-specific misclassification rate $\mathbb{P}\left(Y\not\in \phi\left(\bm X\right)\mid Y=k\right)$, with the latter being the type I error rate for the hypothesis test. The efficiency is inversely measured by the ambiguity, defined as the cardinality of the prediction set $|\phi(\bm x)|=\sum_{k=1}^K\mathbbm{1}\{\bm x\in\mathcal{C}_k\}$. A set-valued classifier with $|\phi(\bm x)|\equiv K$ everywhere is always correct, but contains no useful information. In practice, one may want to balance the two metrics and obtain a classifier with both high accuracy and high efficiency, for example,
\begin{equation*}
\min_{\phi}~ \mathbb{E}\left[|\phi(\bm X)|\right], 
~~~\mbox{s.t.}~  \mathbb{P}\left(Y\not\in\phi(\bm X)\mid Y=k\right)\leq \gamma_k ~\mbox{for}~ k\in[K].
\end{equation*}
In this paper we drop the common restriction  $\mathcal{X}=\cup_{k=1}^K\mathcal{C}_k$ \citep{lei2014classification, sadinle2019least, wang2018learning}. This means that it is possible for $\phi(\bm x)$ to be empty for certain observations, that is, $\bm x$ is unlike any of the existing classes represented in the training data. Note that observations with $|\phi(\bm x)|=0$ and $|\phi(\bm x)|>1$ correspond to outlier observations and ambiguity-rejected observations, respectively. 

The Classification with Reject Options (CRO) literature typically consider the ambiguity-rejections \citep{ramaswamy2015consistent,zhang2018reject} only. \citet{herbei2006classification} and \citet{ramaswamy2015consistent} used 0-$d$-1 loss to quantify the loss for different prediction errors. For example, each misclassification costs 1 and each rejection costs a pre-specified $d\in[0, (K-1)/K]$. The Bayes optimal rule \citep{chow1970optimum} under the 0-$d$-1 loss predicts label $k$ to $\bm x$ if $k=\argmax_{k'}\mathbb{P}(Y=k'\mid \bm x)$ and $\mathbb{P}(Y=k\mid \bm x) > 1-d$, or rejects to predict $\bm x$ otherwise. The plug-in classifier is obtained by first estimating $\mathbb{P}(Y=k\mid \bm x)$ and then plugging them into the Bayes optimal rule. \citet{bartlett2008classification} used the bent hinge loss as a surrogate to the 0-$d$-1 loss and proved the Fisher consistency. \citet{zhang2018reject} introduced a refine option to consider prediction sets with the cardinality between 1 and $K$.

\section{Methodology}\label{sec:method}
We first formulate the proposed GPS method as an optimization problem, which is decoupled into several sub-problems. To solve each sub-problem, we use kernel learning to find a decision function based on the training data from each of the $K$ classes and the test data. Finally, we extend the method to kernel feature selection, to improve its performance for high-dimensional data.

\subsection{Overview of methodology}\label{sec:model}

Suppose our training sample and test sample are two i.i.d. samples from distribution $\mathcal{P}$ and distribution $\mathcal{Q}$, respectively. There are $K$ existing classes in both $\mathcal{P}$  and $\mathcal{Q}$; in addition, there are potentially new classes in $\mathcal{Q}$. Except for the new classes, the two distributions are only different in their prior probabilities for the $K$ classes.

\begin{assumption}\label{assum:1}
For each $k\in[K]$, the probability densities of $\bm X$ given class $k$, $p_{k}(\bm x)=p(\bm x\mid Y=k), $ are the same between distribution $\mathcal{P}$ and distribution $\mathcal{Q}$.
\end{assumption}

Let $\phi(\cdot)$ be a set-valued classifier. Our goal is to maximize both efficiency and accuracy of $\phi(\cdot)$ for future test data drawn from $\mathcal{Q}$. To this end, consider minimizing the ambiguity with class-specific misclassification rates bounded: \begin{equation}\label{eq:joint}
\min\limits_{\phi}~ \mathbb{E}_{\mathcal{Q}}\left[|\phi(\bm X)|\right], ~~~ \mbox{s.t.}~ \mathbb{P}_{\mathcal{Q}}\left(Y\not\in\phi(\bm X)\mid Y=k\right)\leq \gamma_k ~\mbox{for}~ k\in[K].
\end{equation}
Since $\mathbb{E}_{\mathcal{Q}}\left[|\phi(\bm X)|\right]$ does not depend on the class label $Y$, it may be assessed using the unlabelled test data from $\mathcal{Q}$. Moreover, since we assume that $p_{k}(\bm x)$ is the same between both distributions, we have that $\mathbb{P}_{\mathcal{Q}}\left(Y\not\in\phi(\bm X)\mid Y=k\right) = \mathbb{P}_{\mathcal{P}}\left(Y\not\in\phi(\bm X)\mid Y=k\right)$. This allows us to make use of the labeled training data from $\mathcal{P}$ to assess the misclassification rate in the constraint.

Recall that the set-valued classifier $\phi(\cdot)$ is defined using all the $K$ acceptance regions:  $\phi(\bm x) =\{k: \bm x\in \mathcal{C}_k\}$. We typically use a decision function $f_k:\mathcal{X}\mapsto \mathbb{R}$ to define $\mathcal{C}_k$, e.g., $\mathcal{C}_k = \left\{\bm x: f_k(\bm x)\geq0\right\}$. Define the size (probability measure) of $\mathcal{C}_k$ as $\mathcal{R}(f_k)=\mathbb{P}_\mathcal{Q}(f_k(\bm X)\geq0)$.
Under these notations, $\mathbb{E}_{\mathcal{Q}}\left[|\phi(\bm X)|\right]=\sum_{k=1}^K\mathcal{R}(f_k)$. Therefore, the optimization (\ref{eq:joint}) can be decoupled to $K$ separate optimization problems: for each $k\in[K]$, we solve
\begin{equation}\label{eq:NP2}
\begin{aligned}
\argmin\limits_{f_k\in\mathcal{F}}~  \mathcal{R}(f_k),~~~\mbox{s.t.}~ \mathcal{R}^+(f_k)\leq \gamma_k,
\end{aligned}
\end{equation}
where $\mathcal{R}^+(f_k)\triangleq\mathbb{P}_\mathcal{Q}(f_k(\bm X)<0 \mid Y=k)$ and $\mathcal{F}$ is a function space for $f_k$. Throughout this article, we consider the case $\gamma_k=\gamma$ for all $k$. Problem (\ref{eq:NP2}) is equivalent to the Neyman-Pearson classification \citep{scott2005neyman, rigollet2011neyman} in which class $k$ is considered as the null class and the test data is the alternative class.

In practice, one aims to estimate $f_k$ and $\phi$ based on labeled training data $\{(\bm x_i, y_i=k)\}_{i\in\mathcal{G}_k}$ along with unlabeled test data $\{\bm x_j\}_{j\in\mathcal{G}_{te}}$, where $\mathcal{G}_k$ (with size $n_k:=|\mathcal{G}_{k}|$) and $\mathcal{G}_{te}$ (with size $m:=|\mathcal{G}_{te}|$) are index sets for observations in class $k$ of the training data and a subset sampled from the test data respectively. The expectations $\mathcal{R}^+(f_k)$ and $\mathcal{R}(f_k)$ can be replaced by their counterparts under the empirical distributions for the training and test data respectively:
\begin{equation}\label{eq:emp}
 \min\limits_{f_k\in\mathcal{F}}~\frac{1}{m}\sum\limits_{j\in\mathcal{G}_{te}}\mathbbm{1}(f_k(\bm x_j)\geq0), ~~~\mbox{s.t.}~  \frac{1}{n_k}\sum\limits_{i\in\mathcal{G}_k}\mathbbm{1}(f_k(\bm x_i)<0)\leq\gamma.
\end{equation}

\subsection{Surrogate loss and kernel learning}\label{sec:hingeloss}
It is challenging to solve problem (\ref{eq:emp}) due to the use of the 0-1 loss in both the objective and the constraint. A common practice is to replace it by a convex surrogate loss function. Here we use hinge loss $\ell(u)=[1-u]_+=\max(0, 1-u)$ to replace $\mathbbm{1}\{u<0\}$ in the constraint of (\ref{eq:emp}); likewise, $\mathbbm{1}\{u\ge 0\}$ in the objective is replaced by $\ell(-u)$. See Figure \ref{fig:lossfun}. In addition, we use penalty function $J(f_k)$ to control the complexity of decision functions so that (\ref{eq:emp}) becomes:
\begin{equation}\label{eq:surrNP}
  \min\limits_{f_k\in\mathcal{F}}~ \frac{1}{m}\sum\limits_{j\in\mathcal{G}_{te}}[1+f_k(\bm x_j)]_++\lambda J(f_k),
  ~~~\mbox{s.t.}~\frac{1}{n_k}\sum\limits_{i\in\mathcal{G}_{k}}[1-f_k(\bm x_i)]_+\leq\gamma.
\end{equation}
Here the decision function takes the form of $f_k(\bm x)=\bm w_k^\top\Phi(\bm x)-\rho_k$. The first term $\bm w_k^\top\Phi(\cdot)$ belongs to a Reproducing Kernel Hilbert Space (RKHS) $\mathcal{H}$ associated with kernel function $K(\cdot, \cdot)$ and $\Phi$ satisfies $K(\bm x, \bm x^\prime)=\langle \Phi(\bm x), \Phi(\bm x^\prime)\rangle$ for any $\bm x, \bm x^\prime\in\mathcal{X}$. Popular choices of the kernel function include the linear kernel, polynomial kernel, and Gaussian kernel \citep{shawe2004kernel}.

Moreover, following the common practice in the anomaly detection literature \citep{jumutc2013supervised, scholkopf2018learning, shilton2020multiclass}, the penalty $J(f_k)$ is taken as $\frac{1}{2}\|\bm w_k\|^2-\rho_k$. Minimizing $-\rho_k$ encourages a small acceptance region for class $k$, by noting that the acceptance region is $\{\bm x: \bm w_k^\top\Phi(\bm x) \ge \rho_k\}$. As a consequence, this penalty improves the anomaly detection rate. With slackness variables $\eta_i:=[1-\bm w_k^\top\Phi(\bm x_i)+\rho_k]_+$ and $\xi_j:=[1+\bm w_k^\top\Phi(\bm x_j)-\rho_k]_+$, $C:=(\lambda m)^{-1}$, $\Theta:=\{\bm w_k, \rho_k, \{\eta_i\}, \{\xi_j\}\}$, problem (\ref{eq:surrNP}) becomes
\begin{equation}\label{eq:primalprob}
\begin{aligned}
  \min\limits_{\bm \Theta}~ &\frac{1}{2}\|\bm w_k\|^2-\rho_k+C\sum\limits_{j\in\mathcal{G}_{te}}\xi_j,\\
 \mbox{s.t.}~ &\eta_i\geq 1-\bm w_k^\top\Phi(\bm x_i)+\rho_k, ~\xi_j\geq 1+\bm w_k^\top\Phi(\bm x_j)-\rho_k, \\&\sum\limits_{i\in\mathcal{G}_k}\eta_i\leq n_k\gamma, ~ \eta_i\geq 0, ~ \xi_j\geq0.
\end{aligned}
\end{equation}

Following the standard manipulations of the optimization problem using the Lagrange method and the KKT conditions, the dual problem of (\ref{eq:primalprob}) is:
\begin{equation}\label{eq:dualprob}
\begin{aligned}
 \min\limits_{\bm\alpha, \bm\beta, \theta}~
 & \frac{1}{2}\left(\bm\alpha^\top \mathbf G_1\bm\alpha+\bm\beta^\top \mathbf G_2\bm\beta-2\bm\alpha^\top \mathbf G_3\bm\beta\right) -\mathbf{1}_{n_k}^\top\bm\alpha-\mathbf{1}_{m}^\top\bm\beta+n_k\theta\gamma,\\
 \mbox{s.t.}~ &\bm0\preceq\bm\alpha\preceq  \theta\cdot\mathbf{1}_{n_k},~ \bm0\preceq\bm\beta\preceq C\cdot\mathbf{1}_{m},~\mathbf{1}_{n_k}^\top\bm\alpha-\mathbf{1}_{m}^\top\bm\beta =1, ~ \theta\geq0,
\end{aligned}
\end{equation}
 where $\mathbf G_1[i, i^\prime]=K(\bm x_i, \bm x_{i^\prime})$, $\mathbf G_2[j, j^\prime]=K(\bm x_j, \bm x_{j^\prime})$, $\mathbf G_3[i, j]=K(\bm x_i, \bm x_j)$, $\bm\alpha = (\ldots, \alpha_i, \ldots)^\top$, $\bm\beta=(\ldots, \beta_j, \ldots)^\top$, $i, i^\prime\in \mathcal{G}_k, j, j^\prime\in\mathcal{G}_{te}$. This is a quadratic programming (QP) and can be solved with many off-the-shelf packages. Once the dual problem returns minimizers $\hat{\bm\alpha}$ and $\hat{\bm \beta}$, we have $\widehat{\bm w}_k=\sum_{i\in\mathcal{G}_k}\hat\alpha_i\Phi(\bm x_i)-\sum_{j\in\mathcal{G}_{te}}\hat\beta_j\Phi(\bm x_j)$. Finally an estimate to $\rho_k$ can be obtained after plugging $\widehat{\bm w}_k$ back to the primal problem (\ref{eq:primalprob}), which will become a linear programming with respect to $\rho_k, \eta_i,$ and $\xi_j$. The estimated decision function for class $k$ is written as
\begin{equation*}
  \hat f_k(\bm x)=\sum\limits_{i\in\mathcal{G}_k}\hat\alpha_iK(\bm x, \bm x_i)-\sum\limits_{j\in\mathcal{G}_{te}}\hat\beta_j K(\bm x, \bm x_j)-\hat\rho_k,
 \end{equation*}
and the acceptance regions and the set-valued classifier can be obtained accordingly. 

\subsection{Kernel feature selection}
For high-dimensional data, irrelevant or noisy features may degrade our set-valued classifiers' performance in terms of efficiency, accuracy, and outlier detection. Feature or variable selection is necessary in these scenarios. For linear learning, sparse learning using sparsity penalties \citep{tibshirani1996regression, zou2005regularization, zhang2010nearly} has been effective for feature selection. For kernel learning, \citet{allen2013automatic} and \citet{chen2018double} studied weighted kernel feature selection methods. The main idea of these methods is to compute the kernel matrix based on weighted features with a weight vector $\bm d$, and then impose a sparsity-inducing regularization for weight $\bm d$ in the objective function. Adopting this idea, our decision function $f_k$ can be solved using the below optimization problem that enables kernel feature selection:
\begin{equation}\label{eq:KFSprob}
    \begin{aligned}
 \min\limits_{\bm d, \bm\alpha, \rho_k}~ &\frac{1}{m}\sum\limits_{j\in\mathcal{G}_{te}}\ell\left(-f_k(\bm d\circ \bm x_j)\right)+\lambda_1J(f_{k}(\bm d \circ \cdot))+\lambda_2\|\bm d\|_1,\\ 
\mbox{s.t.}~& \frac{1}{n_k}\sum\limits_{i\in\mathcal{G}_{k}}\ell\left(f_k(\bm d\circ \bm x_i)\right)\leq\gamma, \ \bm 0\preceq\bm d\preceq \mathbf 1,
\end{aligned}
\end{equation}
where $\circ$ stands for the Hadamard product.

Our decision function $f_k(\bm d\circ \bm x)=g_k(\bm d\circ \bm x)-\rho_k$. The first term $g_k(\bm d\circ \cdot)$ comes from a Reproducing Kernel Hilbert space (RKHS) associated with kernel function $K_{\bm d}(\cdot, \cdot)$. Here we define $K_{\bm d}(\bm x_i, \bm x_j):=K(\bm d\circ \bm x_i,\bm d\circ \bm x_j)$. By the Representer theorem \citep{kimeldorf1971some}, for some $\alpha_i$ and $\rho_k$, the minimizer to (\ref{eq:KFSprob}) satisfies $$\hat f_k(\bm x)=\sum\limits_{i=1}^{n_k+m}\alpha_iK_{\bm d}( \bm x, \bm x_i)-\rho_k.$$ Moreover, the model complexity function is taken as $$J(f_{k}(\bm d \circ \cdot))=\frac{1}{2}\sum_{i,j=1}^{n_k+m}\alpha_i\alpha_jK_{\bm d}( \bm x_i, \bm x_j)-\rho_k.$$ Denote the kernel matrix $\mathbf{K}_{\bm d}$ with $\mathbf{K}_{\bm d}[i, j]:=K(\bm d\circ\bm x_i, \bm d\circ\bm x_j)$. Let $C_1:=(\lambda_1 m)^{-1}$ and $C_2:={\lambda_2}/{\lambda_1}$. Then we rewrite (\ref{eq:KFSprob}) as
\begin{equation}\label{eq:KFSprob2}
    \begin{aligned}
     \min\limits_{\bm d, \bm\alpha, \rho_k} &\frac{1}{2}{\bm \alpha}^\top \mathbf{K}_{\bm d}{\bm \alpha}-\rho_k + C_1\sum\limits_{j\in\mathcal{G}_{te}}\ell\left(\rho_k-\mathbf{K}_{\bm d}[j,:]\bm\alpha\right)+C_2\|\bm d\|_1,\\
\mbox{s.t.}~ & \frac{1}{n_k}\sum\limits_{i\in\mathcal{G}_{k}}\ell\left(\mathbf{K}_{\bm d}[i,:]\bm\alpha-\rho_k\right)\leq\gamma, ~ \bm 0\preceq\bm d\preceq \mathbf 1.
    \end{aligned}
\end{equation}
Neither the objective nor the first constraint in (\ref{eq:KFSprob2}) is convex with respect to $(\bm d, \bm\alpha, \rho_k)$ despite the convex surrogate loss function (which we chose to be the hinge loss). To resolve this issue, we use an iterative approach by alternatively fixing $\bm d$ while optimizing with respect to $(\bm\alpha, \rho_k)$, which amounts to convex optimization, and fixing $(\bm\alpha, \rho_k)$ while optimizing with respect to $\bm d$. The latter optimization is still not convex. But we can use a linear approximation of the kernel matrix with respect to $\bm d$ to make it convex \citep{zou2008one,lee2012multiple}. In particular, we approximate the kernel matrix by expanding it at $\bm d^\prime$:
\begin{equation*}
    \mathbf K_{\bm d}[i, j]\approx \mathbf K_{\bm d^\prime}[i, j]+\nabla \mathbf K_{\bm d^\prime}[i, j]^\top(\bm d-\bm d^\prime).
\end{equation*}

Define an $(n_k+m)\times (n_k+m)$ matrix $\mathbf A_{\bm d^\prime}$ with $\mathbf A_{\bm d^\prime}[i, j]=\mathbf K_{d^\prime}[i,j]-\nabla \mathbf K_{\bm d^\prime}[i, j]^\top\bm d^\prime$ and a $p \times (n_k+m)$ matrix $\mathbf B_{\bm\alpha}$ with $\mathbf B_{\bm\alpha}[:, i]=\sum_{j=1}^{n_k+m}\alpha_j\nabla \mathbf K_{\bm d^\prime}[i,j]$, where $p$ is the dimension of the data. These allow to approximate (\ref{eq:KFSprob2}) with $(\bm \alpha,\rho_k)$ fixed:
\begin{equation}\label{eq:KFSprob3}
\begin{aligned}
\min\limits_{\bm d}~ &
\frac{1}{2}\bm d^\top \mathbf B_{\bm\alpha}\bm\alpha +C_1\sum\limits_{j\in\mathcal{G}_{te}}\ell\left(\rho_k-\mathbf A_{\bm d^{\prime}}[j,:]\bm\alpha-\mathbf B_{\bm\alpha}[:,j]^\top\bm d\right)+C_2\|\bm d\|_1,
\\
\mbox{s.t.}~ &\frac{1}{n_k}\sum\limits_{i\in\mathcal{G}_k}\ell\left(\bm A_{\mathbf d^{\prime}}[i,:]\bm\alpha+\mathbf B_{\bm\alpha}[:, i]^\top\bm d-\rho_k\right)\leq\gamma,~\bm0\preceq\bm d\preceq\mathbf1.
\end{aligned}
\end{equation}
The above optimization is convex with respect to $\bm d$. The pseudocode of our method is outlined in \cref{alg1} (see the \cref{alg:outline}). 

After we have obtained the decision function $\hat f_k(\bm d\circ\cdot)$ for $k\in[K]$ from Algorithm \ref{alg1}, a set-valued classifier can be constructed as $\hat \phi(\bm x)=\{k\in[K]: \hat f_k(\bm d\circ\bm x)\geq0\}$. If $\hat f_k(\bm d\circ\bm x)<0$ for all $k\in[K]$ for some $\bm x$, then $\bm x$ is detected as an outlier.

\section{Statistical learning theory}\label{sec:theory}
In this section, we study the theoretical properties of our proposed classifier. We will focus on the kernel learning setting. Without loss of generality, we consider the decision function for class 1. For simplicity, we abuse the notation slightly by letting $f(\bm x)=f_1(\bm d \circ\bm x)$, omitting the weight $\bm d$.

Let $f$ be an element from the hypothesis space defined as $\mathcal{F}_{s, s^\prime}=\{f:f(\bm x)=g(\bm x)-\rho,~ g\in \mathcal{H}_{K_{\bm d}}, ~J(f)\leq s^2, ~\|\bm d\|_1\leq s^\prime, ~\bm0\preceq\bm d\preceq\mathbf1 \}$.
Let a subspace that contains decision functions with bounded class 1 misclassification error to be $    \mathcal{F}^+_{s, s^\prime}(\gamma)=\left\{f\in \mathcal{F}_{s,s^\prime}:\mathbb{E}_\mathcal{Q}\left[\ell(f(\bm X))\mid Y=1\right]\leq\gamma\right\}$,
and let its empirical counterpart to be $\widehat{\mathcal{F}}^+_{s,s^\prime}(\gamma)=\{f\in \mathcal{F}_{s, s^\prime}:\frac{1}{n_1}\sum_{i\in\mathcal{G}_1}\ell(f(\bm x_i))\leq\gamma\}$. We consider an optimization problem equivalent to that of  (\ref{eq:KFSprob}) by moving the penalties $J(f)$ and $\|\bm d\|_1$ to the constraints.
Specifically, we consider
\begin{equation}\label{eq:empoptfk}
\argmin_{f\in \widehat{\mathcal{F}}^+_{s,s^\prime}(\gamma)} \frac{1}{m}\sum\limits_{j\in\mathcal{G}_{te}}\ell\left(-f(\bm x_j)\right).
\end{equation}

Denote $\mathbb{P}\left(f(\bm X)\ge 0\mid Y = 1\right)$ and $\mathbb{E}\left[\ell(f(\bm X))\mid Y=1\right]$ as the risk functions of class 1 data under the 0-1 loss and the $\ell$ loss, respectively. Theorem \ref{thm:type1bnd} shows one can bound the former by controlling the empirical counterpart of the latter.

\begin{theorem}\label{thm:type1bnd}
Assume $\kappa=\sup_{\bm x\in\mathcal{X}}\sqrt{\|K_{\bm d}(\bm x, \bm x)\|}$, and the loss function $\ell$ in (\ref{eq:empoptfk}) has a sub-derivative bounded by $c:=\sup_{\bm x\in\mathcal{X}}|\ell^\prime(\bm x)|$. Let $\hat f$ be a solution to (\ref{eq:empoptfk}). With probability $1-\zeta$ over the training sample (incl. $\mathcal{G}_1$ and $\mathcal{G}_{te}$), we have
\begin{equation}\label{eq:thm1}
  \mathbb{E}_{\mathcal{Q}}\left[\ell(\hat f(\bm X))\mid Y=1\right]\leq\frac{1}{n_1}\sum\limits_{i\in\mathcal{G}_1}\ell(\hat f(\bm x_i))+\frac{r(\zeta)}{\sqrt{n_1}},
\end{equation}
where $\left\{(\bm x_i, y_i)\right\}_{i=1}^{n_1}$ are sampled from $\mathbb{P}_\mathcal{Q}[\ \cdot\mid Y=1]$, and $r(\zeta)=(\sqrt{2}s+2)c\kappa\cdot\left(2+3\sqrt{2\log\frac{2}{\zeta}}\right)$.
\end{theorem}

For the Gaussian kernel, $\kappa=1$. Theorem \ref{thm:type1bnd} applies to any convex loss function $\ell$ bounded from below by the 0-1 loss with a Lipschitz constant $c$ satisfying $|\ell(u_1)-\ell(u_2)|\leq c|u_1-u_2|$ for any $u_1$ and $u_2$. In particular, $c=1$ for the hinge loss, the Huberized squared hinge loss, and the logistic loss; the exponential loss has a Lipschitz constant only when the input space is bounded.

Bounding the empirical $\ell$-risk $\frac{1}{n_1}\sum_{i=1}^{n_1}\ell(\hat f(\bm x_i))$ by $\gamma$ may still lead to $\mathbb{E}_{\mathcal{Q}}\bigl[\ell(\hat f(\bm X))\mid Y=1\bigr]$ exceeding $\gamma$. Hence, to better control the true misclassification rate, one can strengthen the constraint by bounding $\frac{1}{n_1}\sum_{i=1}^{n_1}\ell(\hat f(\bm x_i))$ by $\gamma-\varepsilon$ with $\varepsilon=r(\zeta)/\sqrt{n_1}$.

Let the $\ell$-ambiguity be $\mathcal{R}_\ell(f):=\mathbb{E}_{\mathcal{Q}}[\ell(-f(\bm X))]$. Theorem \ref{thm:estmationerror} shows how do the sample size and hypothesis space affect the convergence of the estimation error $\mathcal{R}_\ell(\hat f)-\inf_{f\in\mathcal{{F}}^+_{s,s^\prime}(\gamma)}\mathcal{R}_\ell(f)$.

\begin{theorem}\label{thm:estmationerror}
Under the assumption in Theorem \ref{thm:type1bnd} with Huberized squared hinge loss, let $$\hat f=\argmin\limits_{f\in \widehat{\mathcal{F}}^+_{s,s^\prime}(\gamma-\varepsilon)} \frac{1}{m}\sum\limits_{j\in\mathcal{G}_{te}}\ell\left(-f(\bm x_j)\right).$$ With probability $1-2\zeta$, we have\\
(1) $\mathbb{E}_{\mathcal{Q}}\left[\ell(\hat f(\bm X))\mid Y=1\right]\leq\gamma$,~\mbox{and}~
(2)
$\displaystyle
\mathcal{R}_\ell(\hat f)-\inf\limits_{ f\in\mathcal{{F}}^+_{s,s^\prime}(\gamma)} \mathcal{R}_\ell(f)\leq
\frac{2r(\zeta)}{\sqrt{m}}+\frac{(4+\delta)r(\zeta)}{\sqrt{n_1}\gamma-2r(\zeta)}$.
\end{theorem}

In order to ensure an estimation $\hat f\in\mathcal{{F}}^+_{s,s^\prime}(\gamma)$, by Theorem \ref{thm:type1bnd}, we restrict the hypothesis space as $\mathcal{\widehat{F}}^+_{s,s^\prime}(\gamma-\varepsilon)$. In this setting, the estimation error converges at a rate of $O(\frac{1}{\sqrt{m}}+\frac{1}{\sqrt{n_1}})$. This indicates it is possible for the empirical $\ell$-ambiguity to converge to its minimum in a given hypothesis space using our method. Note that Theorem \ref{thm:estmationerror} also applies to the hinge loss (where $\delta=0$).

Proposition \ref{thm:excessbnd} allows to bound the excess ambiguity by the excess $\ell$-ambiguity.
\begin{proposition}
\label{thm:excessbnd}
\citep{rigollet2011neyman} Let $\mathcal{R}(\cdot)$ and $\mathcal{R}^+(\cdot)$ be defined using the 0-1 loss as in (\ref{eq:NP2}). Given any function $\tilde f$, the following inequality holds
\begin{equation*}
\mathcal{R}(\tilde f)-\inf\limits_{\mathcal R^+(f)\leq\gamma} \mathcal R(f)\leq
\mathcal{R}_\ell(\tilde f)-\inf\limits_{\mathcal R^+(f)\leq\gamma} \mathcal R_\ell(f).\\
\end{equation*}
\end{proposition}

Proposition \ref{thm:excessbnd} shows that we can control the excess ambiguity by controlling the excess $\ell$-ambiguity using a good estimate $\hat f$.

Let $\bm d^*=(d^*_t)$ be the weight in $f^\ast\in\arginf_{f\in\mathcal{{F}}^+_{\infty,1}(\gamma)}\mathcal R_\ell(f)$. The important and unimportant features are referred as those $\bm x_{\cdot, t}$ with $d^*_t>0$ and $d^*_t=0$, respectively. Theorem \ref{thm:featsel} shows feature selection consistency in terms of the sign of weight $\widehat{\bm d}$ under some conditions.

\begin{theorem}\label{thm:featsel} Consider the hypothesis space as a Gaussian kernel RKHS and the input space $\mathcal{X}$ is bounded. Let a Lipschitz continuous loss function $\ell(u)$ be  differentiable, and $\widehat{\bm d}=(\hat d_t)$ be the solution to (\ref{eq:empoptfk}).  Assume $f^*$ has a sparse representation in the RKHS, and $\left.\frac{\partial\mathbb{E}_{\mathcal{Q}}[\ell(-f^\ast(\bm X))]}{\partial d_t}\right|_{d_t=0,~ d_{t^\prime}=d^\ast_{t^\prime}, ~\forall t^\prime\neq t}$ 
are negative and non-negative for those important and unimportant features $\bm x_{\cdot, t}$, respectively, then
\begin{equation*}
\mathbb{P}\left[\sign(\hat d_t)=\sign(d^\ast_t)\right]\longrightarrow1, \ t\in[p].
\end{equation*}
\end{theorem}

Under the assumption for the partial derivatives, the optimization procedure in (\ref{eq:empoptfk}) will lead to a solution where the weight $\hat d_t>0$ for important features and $=0$ for unimportant features, respectively, for a large enough sample. Similar assumptions were used in \citet{fan2004nonconcave} and \citet{fan2010selective}.

\section{Numerical studies}\label{sec:numerical}
For the GPS methods, we use \texttt{cvxopt} or \texttt{scipy} in Python to solve the convex optimization problems involved. For competing methods like kernel density estimation (KDE) \citep{hechtlinger2018cautious}, OCSVM \citep{scholkopf2000support} and BSVM \citep{steinwart2005classification, du2014analysis, du2015convex}, we use their implementations in the  \texttt{scikit-learn} library. For the BCOPS-RF method \citep{guan2019prediction} involving Random Forest, we implemented it based on the \texttt{ensemble.RandomForestClassifier} function in the  \texttt{scikit-learn} library. 
We report the coverage rate 
${\left|j\in\mathcal{G}_{te}:Y_j=k~\&~Y_j\in \phi(\bm x_j)\right|}/\left|j\in\mathcal{G}_{te}:Y_j=k\right|$, the average cardinality of the decision rule $\frac{1}{m}\sum_{j\in\mathcal{G}_{te}}\left|\phi(\bm x_j)\right|$, the average (conditional) cardinality given non-outliers $\sum_{j\in\mathcal{G}_{te}}(\left|\phi(\bm x_j)\right|\cdot\mathbf{1}\{Y_j\neq \text{Outlier}\})/\sum_{j\in\mathcal{G}_{te}}\mathbf{1}\{Y_j\neq \text{Outlier}\}$, and the outlier detection rate $\sum_{j\in\mathcal{G}_{te}}\mathbf{1}\{Y_j= \text{Outlier}~\&~|\phi(\bm x_j)| =0\}/\sum_{j\in\mathcal{G}_{te}}\mathbf{1}\{Y_j= \text{Outlier}\}$. We report the average of these metrics  after 200 replications on another subset of the test data not used in model training.

\begin{figure}[!h]
\begin{minipage}[h]{0.32\textwidth}
\centering
\includegraphics[scale = 0.25]{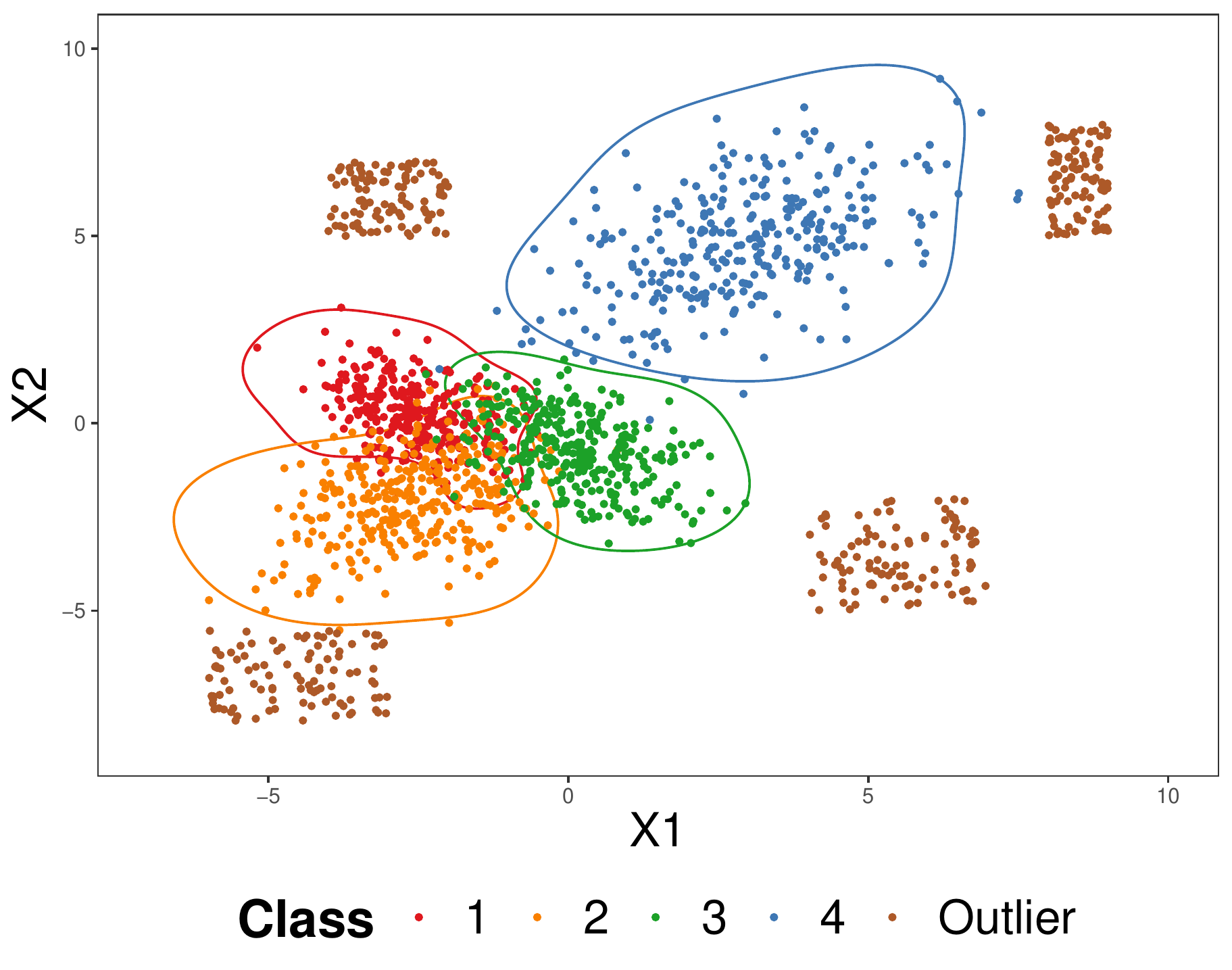}
\end{minipage}\hspace{1pt}
\begin{minipage}[h]{0.32\textwidth}
\centering
\includegraphics[scale = 0.25]{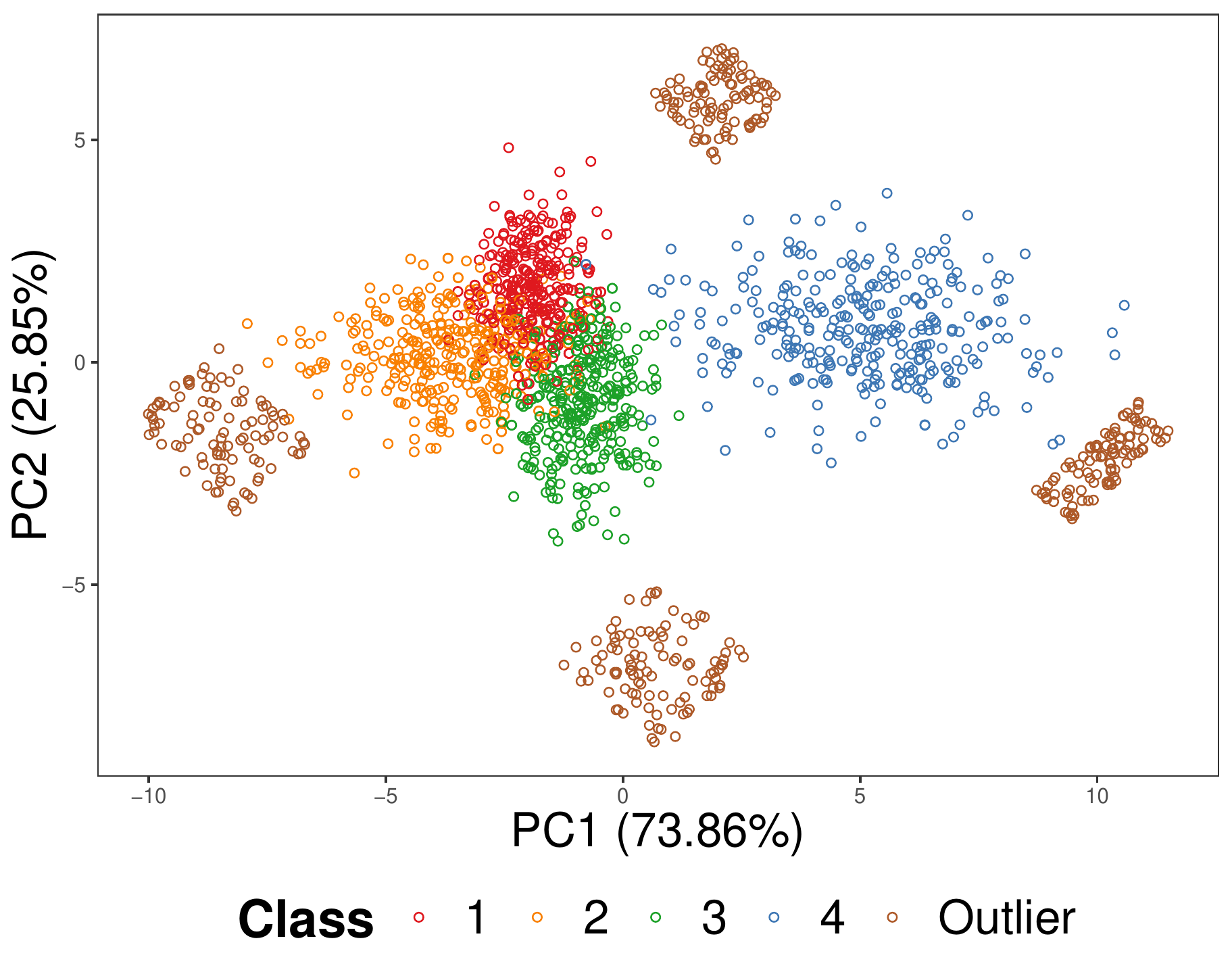}
\end{minipage}\hspace{1pt}
\begin{minipage}[h]{0.32\textwidth}
\centering
\includegraphics[scale = 0.25]{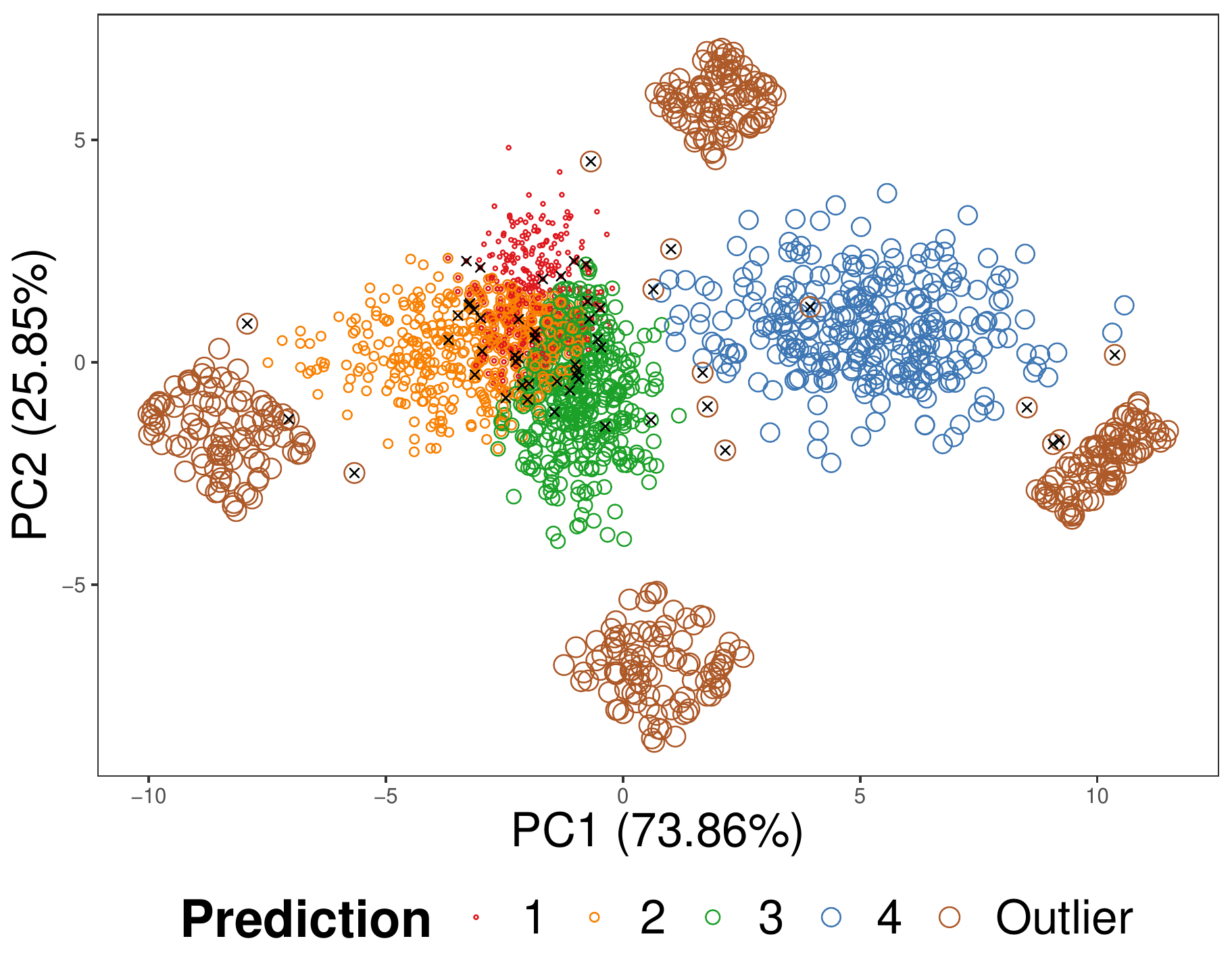}
\end{minipage}
\vskip -0.in
\caption{Example 1 and results of GPS classification. Left panel: The scatter plot for the first two dimensions. Colored contours are boundaries of acceptance regions for each of the four classes with $\gamma=5\%$. Middle panel: The scatter plot of the first two principle components for the test data. Right Panel: Same as the middle panel with the color of the circles indicating the predicted set of classes. The radius of the circles differ for different classes to allow the visualization of those points which are classified into multiple classes. Points whose true labels are not contained in the prediction set are labeled as black crosses. Points with empty prediction sets are marked as brown, that is, the outlier class.}
\label{fig:ex1countour}\vskip -0.5in
\end{figure}

\subsection{Simulations}\label{simulation}
\textbf{Example 1:} We first generate data points from the outlier class which is a mixture of four uniform distributions (with equal weights) on four rectangle regions as shown in the left panel of \cref{fig:ex1countour} (see the \cref{app:simdt}). Then we generate data from four multivariate normal classes ($k=1,\dots,4$) where $\bm X \mid Y=k\sim\mathcal{N}(\bm \mu_k, \Sigma_k)$, $\bm \mu_k=r_k(\cos\theta_k, \sin\theta_k)^\top$, $r_k\sim \text{Uniform}(0,6)$, $\theta_k\sim \text{Uniform}(0, 2\pi)$ and $\Sigma_k^{1/2}=\text{diag}(\sigma_k, \sigma_k) + \varepsilon_k$, where $\sigma_k\sim\text{Uniform}(0.8, 1.2)$ and $\varepsilon_k\sim\text{Uniform}(-0.5, 0.5)$. After generating the above two-dimensional data, we augment them with eight independent noise variables which are normal distributed with mean 0 and standard deviation 0.1.

\begin{table}[!b]
 \setlength{\abovecaptionskip}{-5pt} 
  \caption{Average performance metrics for Example 1}\label{tab:ex1AvgPerm}
  \vskip 0.1in
  \centering
  \begin{adjustbox}{max width=1.\textwidth}
  \begin{tabular}{cccccccc}
  \toprule
    &   & KDE & OCSVM & BSVM & BCOPS-RF & GPS & GPSKFS\\
  \midrule
   & Class 1 & \makecell[c]{0.963432\\(0.001161)} & \makecell[c]{0.964854\\(0.000852)} & \makecell[c]{0.962592\\(0.001013)} & \makecell[c]{0.9601\\(0.000926)} & \makecell[c]{0.959798\\(0.001155)} & \makecell[c]{0.954419\\(0.001283)}\\
  
   & Class 2 & \makecell[c]{0.95531\\(0.001139)} & \makecell[c]{0.954323\\(0.001054)} & \makecell[c]{0.966999\\(0.000874)} & \makecell[c]{0.966385\\(0.000941)} & \makecell[c]{0.957702\\(0.000931)} & \makecell[c]{0.953851\\(0.001021)}\\
  
   & Class 3 & \makecell[c]{0.95395\\(0.001079)} & \makecell[c]{0.963451\\(0.000816)} & \makecell[c]{0.960052\\(0.000807)} & \makecell[c]{0.95673\\(0.0011)} & \makecell[c]{0.959188\\(0.000976)} & \makecell[c]{0.954885\\(0.001104)}\\
  
  \multirow[t]{-4}{*}{\centering\arraybackslash \makecell[c]{Cvg.\\ rate}} & Class 4 & \makecell[c]{0.952608\\(0.001166)} & \makecell[c]{0.961199\\(0.000866)} & \makecell[c]{0.961176\\(0.001049)} & \makecell[c]{0.962935\\(0.001001)} & \makecell[c]{0.956328\\(0.001053)} & \makecell[c]{0.952011\\(0.001275)}\\
  \cmidrule{1-8}
  \multicolumn{2}{c}{Card.} & \makecell[c]{1.229944\\(0.002731)} & \makecell[c]{1.180821\\(0.002715)} & \makecell[c]{1.024036\\(0.006351)} & \makecell[c]{1.072269\\(0.003577)} & \makecell[c]{0.97462\\(0.001868)} & \makecell[c]{0.967706\\(0.002969)}\\
  
   \multicolumn{2}{c}{Cond. Card.} & \makecell[c]{1.620047\\(0.003604)} & \makecell[c]{1.562107\\(0.003602)} & \makecell[c]{1.334975\\(0.005168)} & \makecell[c]{1.341478\\(0.002621)} & \makecell[c]{1.288988\\(0.002285)} & \makecell[c]{1.277959\\(0.003488)}\\
  \cmidrule{1-8}
  \multicolumn{2}{c}{Detection rate}  & \makecell[c]{0.976643\\(0.001415)} & \makecell[c]{0.998787\\(0.000234)} & \makecell[c]{0.938637\\(0.010971)} & \makecell[c]{0.767202\\(0.008296)} & \makecell[c]{0.998224\\(0.000381)} & \makecell[c]{0.991713\\(0.003438)}\\
  \bottomrule
  \end{tabular}
  \end{adjustbox}
\end{table}

In \cref{fig:ex1countour}, the set-valued classifier is constructed based on first two dimensional data in left panel. The contours display the boundaries of acceptance regions for 4 classes. We can see the outlier class is successfully ruled out from those acceptance regions. For the middle panel, we project all dimensions onto the first two principle components and distinguish them by their corresponding true labels. Then predictions and detections returned by GPS are visualized by circles with different radius in the right panel. Circles centering at the same observation but with different radius show that the cardinality of prediction set for that observation is more than 1. This means those observations are similar to others and are difficult to be confidently assigned by a single label. The black crosses denote those observations incorrectly detected as outliers or classified as other 3 classes. From the right panel, we can see that this type of decisions are more likely to appear on the tail of class distributions.

We set significance level $\gamma=5\%$ in this simulation. 
\cref{tab:ex1AvgPerm} shows that all methods can control the coverage rate to be at least 95\% for all classes on average.
From the last three rows, we see that the proposed GPS and GPSKFS discover the outlier data with more than $99\%$ probability, being on par with OCSVM and outperforming all other methods. Moreover, the average cardinality and the average conditional cardinality of the prediction sets returned by GPS and GPSKFS are the smallest among all methods.

\textbf{Example 2:} This example is similar to Example 3 in \citet{wang2018learning}. We first generate radius-angle pairs $(R, \theta)$, where $\theta\sim\mbox{Uniform}(0, 2\pi)$. $R\mid Y=1 \sim \mbox{Uniform}(0, 5)$, $R\mid Y=2 \sim \mbox{Uniform}(4, 9)$ and $R\mid Y=3 \sim \mbox{Uniform}(8, 13)$. For the outlier class in the test data, $R\mid \mbox{outlier} \sim \mbox{Uniform}(15, 20)$. Then we define a 2-dimensional data vector $(R\cdot\cos\theta, R\cdot\sin\theta)$. Finally, we add 98 independent normal noise variables with mean 0 and standard deviation 1.

Similar to the data visualization in Example 1, we display the boundaries of the acceptance regions for Example 2 in the left panel of Figure \ref{fig:ex2countour}. We show the first two principal components for the test data, with the true labels in the middle panel and with the predicted label sets in the right panel.

\begin{figure}[!t]
\begin{minipage}[h]{0.32\textwidth}
\adjustimage{scale = 0.25, center}{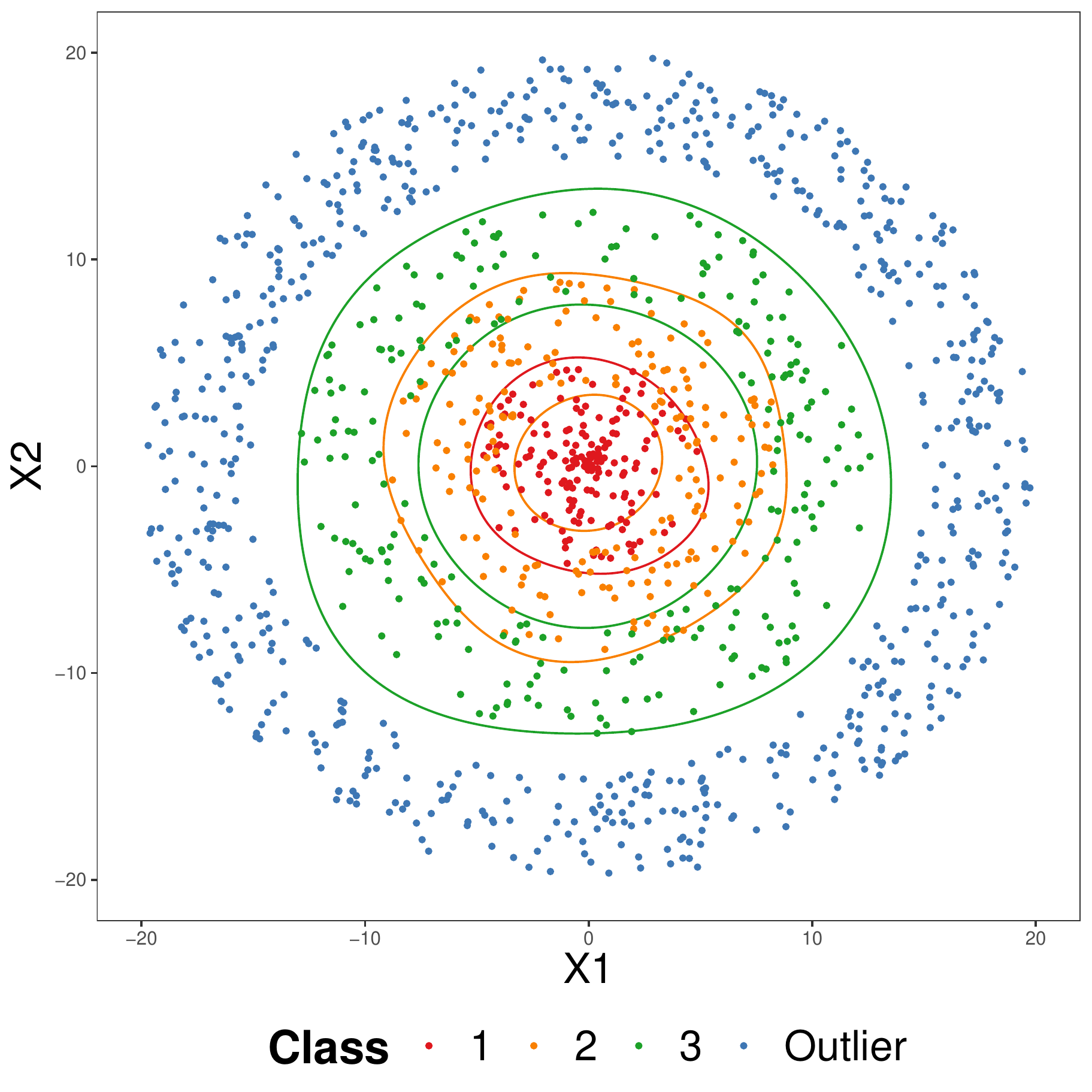}
\end{minipage}\hspace{1pt}
\begin{minipage}[h]{0.32\textwidth}
\adjustimage{scale = 0.25,center}{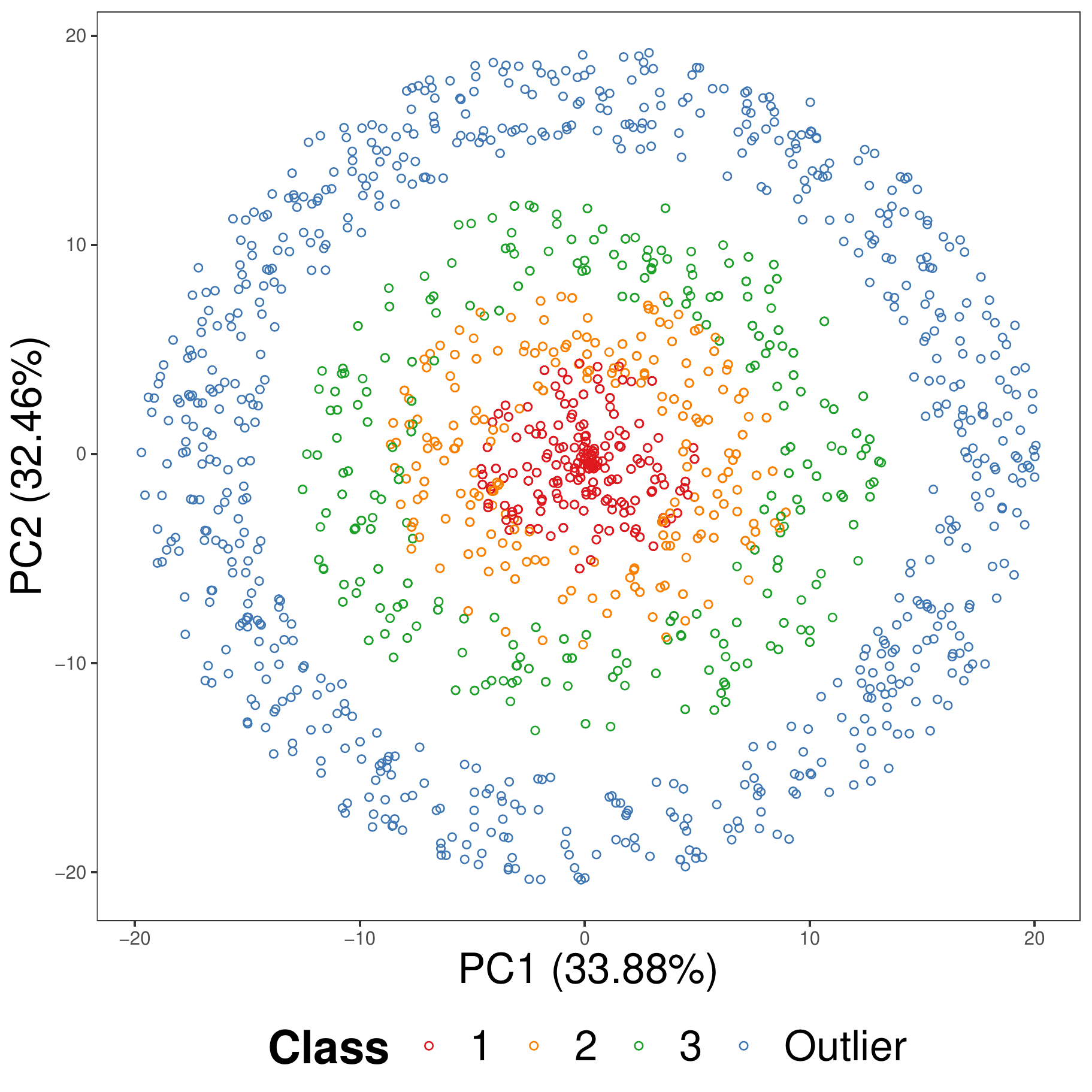}
\end{minipage}\hspace{1pt}
\begin{minipage}[h]{0.32\textwidth}
\adjustimage{scale = 0.25,center}{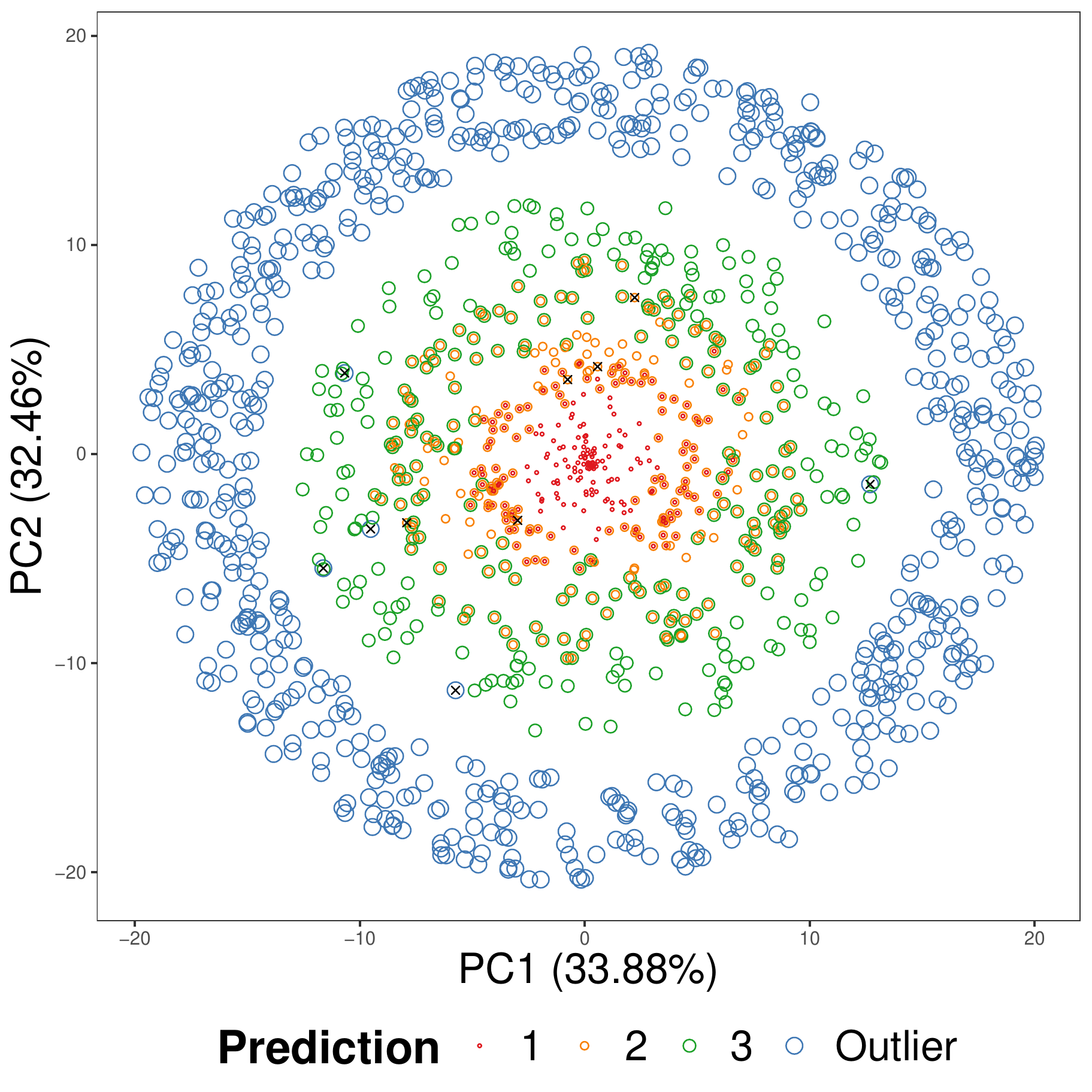}
\end{minipage}
\caption{Example 2 and results of GPS classification. Left panel: The scatter plot for the first two dimensions. Colored contours are boundaries of acceptance regions for each of the four classes with $\gamma=1\%$. Middle panel: The scatter plot of the first two principal components for the test data. Right Panel: Same as the middle panel with the color of the circles indicating the predicted set of classes. The radius of the circles differ for different classes to allow the visualization of those points which are classified into multiple classes. Points whose true labels are not contained in the prediction set are labeled as black crosses. Points with empty prediction sets are marked as blue, the outlier class.}
\label{fig:ex2countour}
\end{figure}

\begin{table}[!h]
  \caption{Average performance metrics for Example 2}
  \label{tab:ex2AvgPerm}
  \vskip 0.1in
  \centering
  \begin{adjustbox}{max width=1.\textwidth}
  \begin{tabular}{cccccccc}
  \toprule
    &   & KDE & OCSVM & BSVM & BCOPS-RF & GPS & GPSKFS\\
  \midrule
   & Class 1 & \makecell[c]{0.987115\\(0.000807)} & \makecell[c]{0.976428\\(0.000959)} & \makecell[c]{0.989296\\(0.000858)} & \makecell[c]{0.986973\\(9e-04)} & \makecell[c]{0.981963\\(0.000926)} & \makecell[c]{0.979427\\(0.000936)}\\
  
   & Class 2 & \makecell[c]{0.990974\\(0.000561)} & \makecell[c]{0.993547\\(0.000353)} & \makecell[c]{0.992637\\(0.000573)} & \makecell[c]{0.994336\\(0.000562)} & \makecell[c]{0.992854\\(0.000419)} & \makecell[c]{0.994052\\(0.000551)}\\
  
  \multirow[t]{-3}{*}{\centering\arraybackslash \makecell[c]{Cvg.\\ rate}} & Class 3 & \makecell[c]{0.988943\\(0.000668)} & \makecell[c]{0.988907\\(0.000546)} & \makecell[c]{0.987811\\(0.000774)} & \makecell[c]{0.993519\\(0.000533)} & \makecell[c]{0.983814\\(0.000703)} & \makecell[c]{0.989801\\(0.00061)}\\
  \cmidrule{1-8}
  \multicolumn{2}{c}{Card.} & \makecell[c]{1.742601\\(0.003386)} & \makecell[c]{1.080905\\(0.002137)} & \makecell[c]{1.163262\\(0.00892)} & \makecell[c]{1.794714\\(0.011305)} & \makecell[c]{1.042262\\(0.002198)} & \makecell[c]{0.623107\\(0.001649)}\\
  
   \multicolumn{2}{c}{Cond. Cardi.} & \makecell[c]{2.654633\\(0.002671)} & \makecell[c]{2.274483\\(0.002598)} & \makecell[c]{2.271912\\(0.006102)} & \makecell[c]{2.483697\\(0.007604)} & \makecell[c]{2.179741\\(0.002661)} & \makecell[c]{1.31847\\(0.002458)}\\
  \cmidrule{1-8}
  \multicolumn{2}{c}{Detection rate} & \makecell[c]{0.333359\\(0.002913)} & \makecell[c]{0.988376\\(0.00045)} & \makecell[c]{0.832284\\(0.018087)} & \makecell[c]{0.263899\\(0.00695)} & \makecell[c]{0.976481\\(0.001266)} & \makecell[c]{1\\(0)}\\
  \bottomrule
  \end{tabular}
  \end{adjustbox}
\end{table}

For this example, we set the significance level $\gamma=1\%$. \cref{tab:ex2AvgPerm} shows that almost all methods can control the coverage rate to be close to 99\% for all classes on average. The last three rows show the efficiency and outlier detection performance of these methods. Both GPS and GPSKFS have significantly smaller cardinalities (although GPS has a slightly smaller detection rate than OCSVM) compared to all other methods, while at the same time, GPSKFS has the striking outlier detection rate. The GPSKFS significantly improves the performance of the regular GPS method.

\begin{figure}[!t]
  \begin{minipage}[h]{1\textwidth}
    \centering
    \includegraphics[width=1\linewidth]{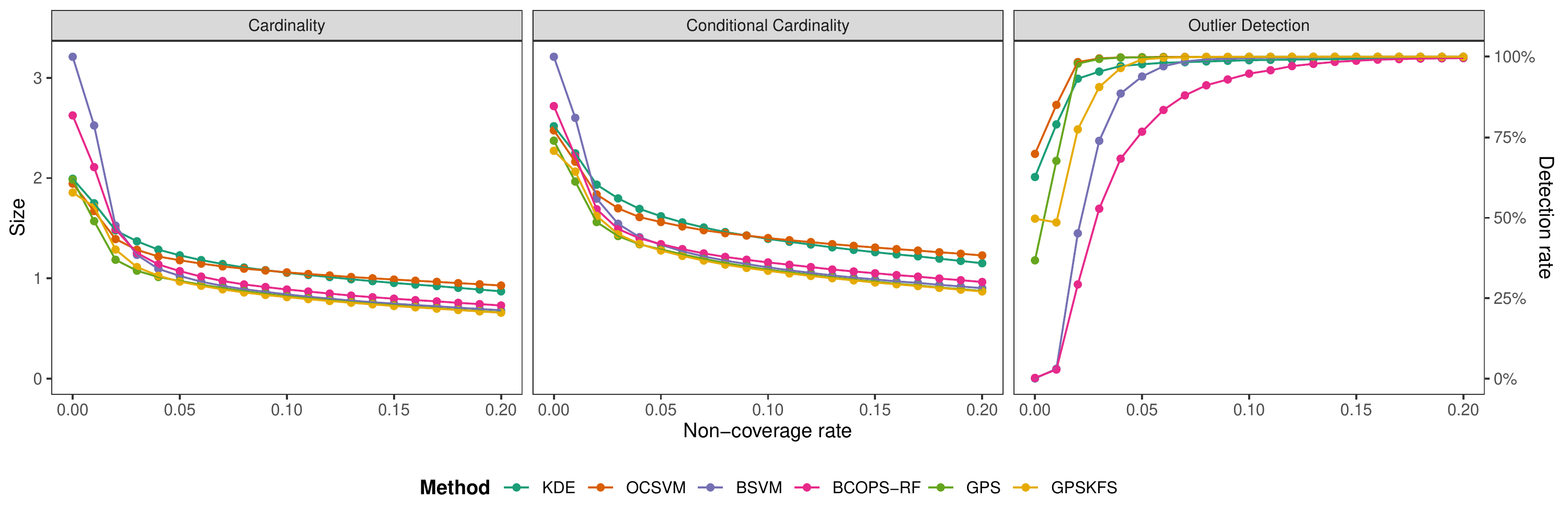}
  \end{minipage}
  \begin{minipage}[h]{1\textwidth}
      \centering
      \includegraphics[width=1\linewidth]{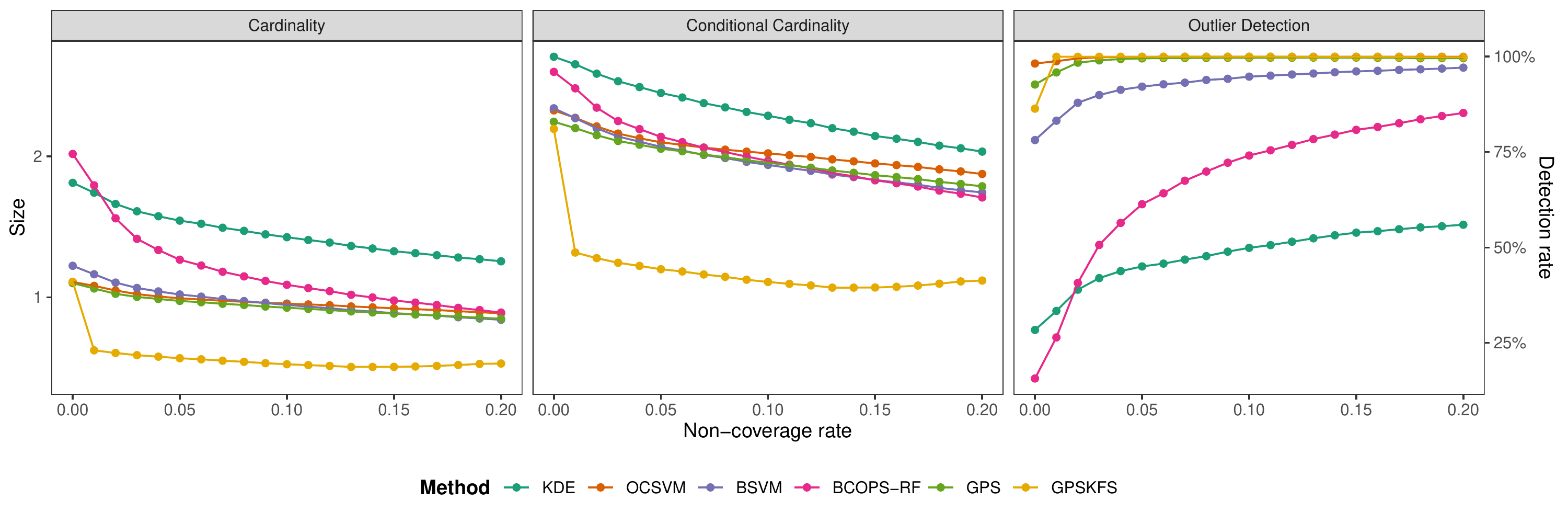}
  \end{minipage}
  \vskip -0.05in
  \caption{Scree plot for different methods in terms of cardinalities and outlier detection rate with different non-coverage rates. The top and bottom panels are for Example 1 and Example 2, respectively. 
  }
  \label{fig:screeplot}
\end{figure}

Note that one should not evaluate rows in \cref{tab:ex1AvgPerm} and \cref{tab:ex2AvgPerm} separately as there exist  trade-offs among non-coverage rate, the average cardinality, and the outlier detection rate, further demonstrated in Figure \ref{fig:screeplot} (Example 1 in the top panel and Example 2 in the bottom). In the two left panels, with increasing non-coverage rates, the average cardinality and the average of conditional cardinality both decrease since the acceptance regions are becoming smaller. Shrinking acceptance regions also allow outlying points to be detected more easily, shown in the right panel. Figure \ref{fig:screeplot} provides a way to choose the non-coverage rate $\gamma$ based on the data. For example, in Example 1, if one cares about the efficiency and the accuracy, then one can identify an elbow in the first two plots (such as $\gamma=0.01$ or $0.02$) which corresponds to $\gamma$ with low cardinality and reasonably low $\gamma$. One should also refer to the third plot to see if the outlier detection performance is acceptable. In Example 2, GPSKFS with a small $\gamma$ such as $0.01$ would be sufficient to obtain an efficient and accurate classifier with a very high outlier detection rate. In practice, one can use a tuning data set or cross-validation to produce these plots.

\begin{table}[!b]
  \caption{Average performance metrics for Zipcode data}\label{tab:tbperfzip}
  \vskip 0.1in
  \centering
  \begin{adjustbox}{max width=1.\textwidth}
  \begin{tabular}{cccccccc}
  \toprule
    &   & KDE & OCSVM & BSVM & BCOPS-RF & GPS & GPSKFS\\
  \midrule
   & Class 1 & \makecell[c]{0.991117\\(0.000361)} & \makecell[c]{0.990761\\(0.000319)} & \makecell[c]{0.993457\\(0.000357)} & \makecell[c]{0.995633\\(0.000257)} & \makecell[c]{0.990907\\(0.000376)} & \makecell[c]{0.986328\\(0.000393)}\\
  
   & Class 2 & \makecell[c]{0.991915\\(0.000543)} & \makecell[c]{0.989583\\(0.000436)} & \makecell[c]{0.990241\\(0.000481)} & \makecell[c]{0.986156\\(0.000518)} & \makecell[c]{0.987301\\(0.000552)} & \makecell[c]{0.982399\\(0.000588)}\\
  
   & Class 3 & \makecell[c]{0.987302\\(0.000664)} & \makecell[c]{0.982153\\(0.000782)} & \makecell[c]{0.985098\\(0.000736)} & \makecell[c]{0.984648\\(0.000698)} & \makecell[c]{0.981102\\(0.000788)} & \makecell[c]{0.969517\\(0.000833)}\\
  
  \multirow[t]{-4}{*}{\centering\arraybackslash \makecell[c]{Cvg.\\ rate}} & Class 4 & \makecell[c]{0.994114\\(0.000258)} & \makecell[c]{0.99707\\(0.000222)} & \makecell[c]{0.992739\\(0.000312)} & \makecell[c]{0.992456\\(0.000408)} & \makecell[c]{0.990116\\(0.000366)} & \makecell[c]{0.985439\\(0.000515)}\\
  \cmidrule{1-8}
  \multicolumn{2}{c}{Card.}& \makecell[c]{2.71731\\(0.01002)} & \makecell[c]{2.332223\\(0.016826)} & \makecell[c]{0.682202\\(0.012149)} & \makecell[c]{0.899256\\(0.0103)} & \makecell[c]{0.621457\\(0.009297)} & \makecell[c]{0.520096\\(0.005579)}\\
  
   \multicolumn{2}{c}{Cond. Card.} & \makecell[c]{2.694657\\(0.010393)} & \makecell[c]{2.41925\\(0.014163)} & \makecell[c]{1.368261\\(0.013598)} & \makecell[c]{1.473345\\(0.009692)} & \makecell[c]{1.262108\\(0.00936)} & \makecell[c]{1.14894\\(0.005256)}\\
  \cmidrule{1-8}
  \multicolumn{2}{c}{Detection rate}  & \makecell[c]{0.032493\\(0.000744)} & \makecell[c]{0.082452\\(0.002549)} & \makecell[c]{0.628362\\(0.007725)} & \makecell[c]{0.469479\\(0.005752)} & \makecell[c]{0.646673\\(0.007275)} & \makecell[c]{0.716897\\(0.004804)}\\
  \bottomrule
  \end{tabular}
  \end{adjustbox}
\end{table}

\subsection{Real data analysis}\label{realdata}
\textbf{Zipcode:} The first real example we consider is a hand-written zipcode dataset which consists of training data with 7291 observations and 256 features as well as test data with 2007 points. We first merge them together and treat labels 0, 6, 8, 9 as the normal classes and the remaining labels as the outlier. To generate our training data, we randomly sample from each of the normal classes with subsample sizes 550, 580, 495, and 574, respectively. The remaining data points from the normal classes, and all points in the outlier class, will form the test data. Following \citet{wang2018learning}, we set the  non-coverage probability as $\gamma=0.01$ in this example.

\noindent
\textbf{Phoneme (Speech Recognition):} This dataset 
with 4509 data observations and 256 features is formed by selecting five phonemes for classification based on digitized speech.  The phonemes are transcribed as follows: ``sh" as in ``she", ``dcl" as in ``dark", ``iy" as the vowel in ``she", ``aa" as the vowel in ``dark", and ``ao" as the first vowel in ``water". We treat class ``sh" as the outlier in test data and sample the other four classes with size around 500 for training data. Here the non-coverage probability is also specified as 0.01 for each class.

\begin{table}[!h]
  \caption{Average performance metrics for Phoneme data}
  \label{tab:tbperfpho}
  \vskip 0.1in
  \centering
  \begin{adjustbox}{max width=1.\textwidth}
  \begin{tabular}{cccccccc}
  \toprule
    &   & KDE & OCSVM & BSVM & BCOPS-RF & GPS & GPSKFS\\
  \midrule
   & Class 1 & \makecell[c]{0.980271\\(0.000719)} & \makecell[c]{0.975932\\(0.000529)} & \makecell[c]{0.992586\\(0.000548)} & \makecell[c]{0.990774\\(0.000483)} & \makecell[c]{0.975624\\(0.000866)} & \makecell[c]{0.98067\\(0.000907)}\\
  
   & Class 2 & \makecell[c]{0.995219\\(0.000397)} & \makecell[c]{0.996666\\(0.000183)} & \makecell[c]{0.99183\\(0.000439)} & \makecell[c]{0.992273\\(0.000398)} & \makecell[c]{0.989256\\(0.000509)} & \makecell[c]{0.983437\\(0.000601)}\\
  
   & Class 3 & \makecell[c]{0.99362\\(0.000452)} & \makecell[c]{0.988917\\(0.000329)} & \makecell[c]{0.990333\\(0.000538)} & \makecell[c]{0.993262\\(0.000521)} & \makecell[c]{0.988117\\(0.000622)} & \makecell[c]{0.983034\\(0.000666)}\\
  
  \multirow[t]{-4}{*}{\centering\arraybackslash \makecell[c]{Cvg.\\ rate}} & Class 4 & \makecell[c]{0.990006\\(0.000504)} & \makecell[c]{0.995385\\(0.000329)} & \makecell[c]{0.994372\\(0.000394)} & \makecell[c]{0.993214\\(0.000363)} & \makecell[c]{0.993802\\(0.000411)} & \makecell[c]{0.985954\\(0.000633)}\\
  \cmidrule{1-8}
  \multicolumn{2}{c}{Card.}& \makecell[c]{1.910706\\(0.005553)} & \makecell[c]{1.4308\\(0.006762)} & \makecell[c]{1.063633\\(0.014423)} & \makecell[c]{1.289214\\(0.016904)} & \makecell[c]{0.935682\\(0.003749)} & \makecell[c]{0.914176\\(0.005313)}\\
  
  \multicolumn{2}{c}{Cond. Card.} & \makecell[c]{2.326539\\(0.004782)} & \makecell[c]{1.967869\\(0.007403)} & \makecell[c]{1.599113\\(0.014544)} & \makecell[c]{1.841759\\(0.017022)} & \makecell[c]{1.459293\\(0.00508)} & \makecell[c]{1.411896\\(0.004725)}\\
  \cmidrule{1-8}
  \multicolumn{2}{c}{Detection. rate} & \makecell[c]{0.250608\\(0.004657)} & \makecell[c]{0.530173\\(0.006858)} & \makecell[c]{0.881095\\(0.014379)} & \makecell[c]{0.684472\\(0.017594)} & \makecell[c]{0.981983\\(0.001574)} & \makecell[c]{0.957807\\(0.009124)}\\
  \bottomrule
  \end{tabular}
  \end{adjustbox}
\end{table}

\cref{tab:tbperfzip} and \cref{tab:tbperfpho} display average performances over 200 replications for the Zipcode and Phoneme data, respectively. From these two tables, we see that the average class-specific coverage rate is close to the desired value. We see KDE and OCSVM not only have the worst outlier detection performance but also return prediction sets with much ambiguity (large cardinality). Additionally, our proposed GPS methods significantly outperform all other methods. GPSKFS returns smaller prediction sets than GPS. In \cref{tab:tbperfpho}, the outlier detection rate from GPSKFS (with feature selection) is smaller than GPS on average, which is due to the extreme outlier rejection performance on several runs (see boxplots with cardinality = 0 in the right panel of \cref{fig:ppsdc}). In general, margin-based classifiers (OCSVM, BSVM, GPS, GPSKFS) outperform the plug-in methods (KDE, BCOPS-RF).

\begin{figure}[!t]
  \adjustimage{scale = 0.55, center}{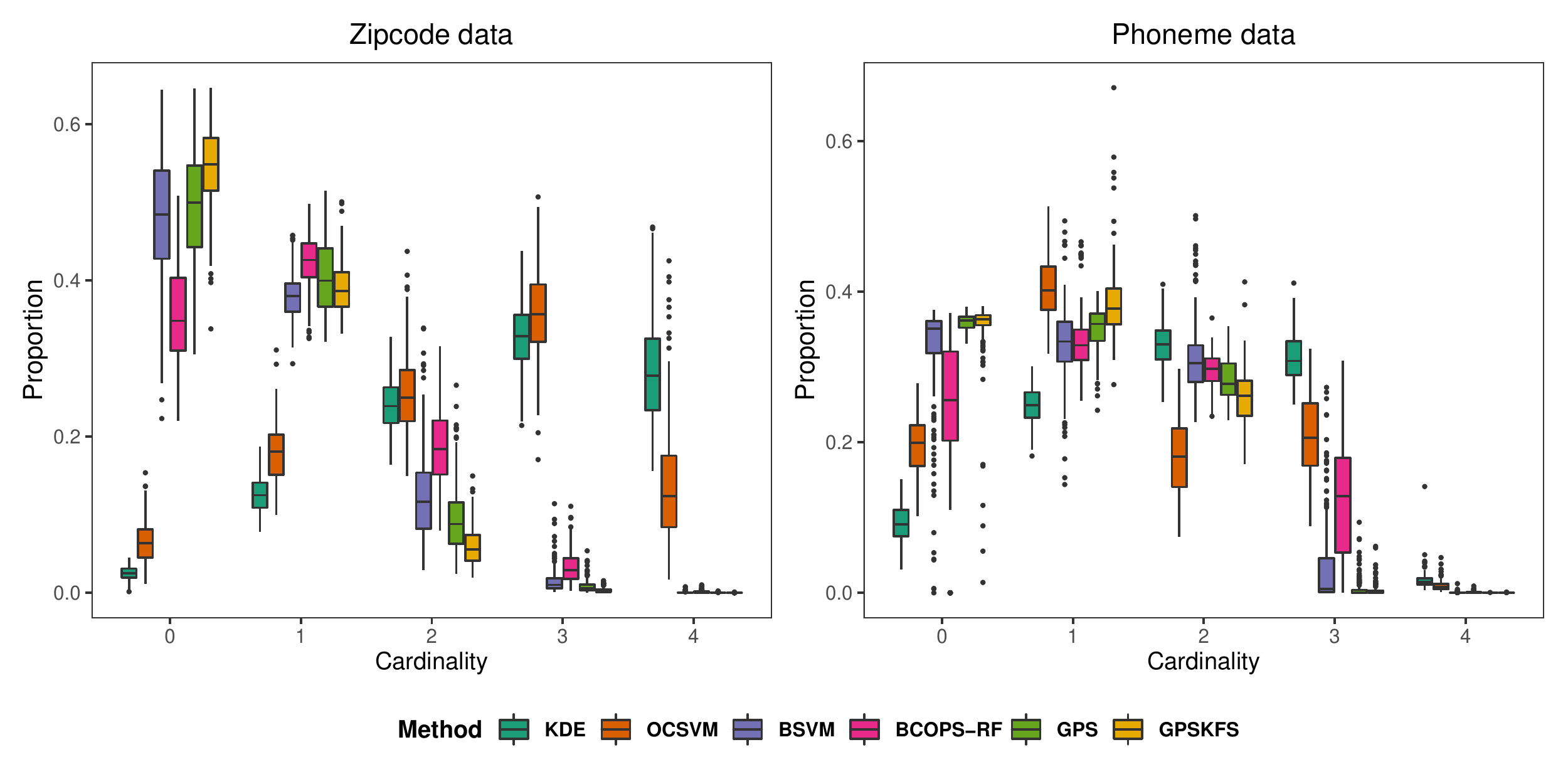}
  \vskip -0.15in
  \caption{Proportions of data points with a different prediction set conditional cardinality.}\label{fig:ppsdc}
\end{figure}

In \cref{fig:ppsdc}, we visualize the distributions of the different observed cardinalities returned by different models in the test data (not used in the training) for both real examples. Observations with zero cardinality are deemed outliers. The larger the cardinality is, the less informative the prediction set is. Generally, we can see KDE and OCSVM have a tendency to return larger and hence less informative prediction sets. In contrast, BSVM, GPS, and GPSKFS produce very few large prediction sets.

In summary, while we do not expect a single method to outperform all other methods in all metrics and situations, we have seen satisfactory results from the GPS methods.

\section{Conclusion}\label{sec:conclusion}

Motivated by the necessity of set-valued predictions with low ambiguity and an outlier detection capacity in critical domains (e.g. in public health), we propose the GPS methods to conduct set-valued multicategory classification with many classes. Our proposed methods are capable of detecting novel classes in the test data that have not appeared in the training data. The GPS methods have a feature selection property in the kernel learning context. We make use of the ``divide-and-conquer" strategy to break down a large-scale problem into many sub-problems, involving only two classes of data, namely an existing class $k$ and the subset of test data which may include a novel class. Because all the sub-problems can be solved in an \textit{embarrassingly parallel} way, the computational time can be greatly reduced compared to solving a large-scale optimization problem involving all classes.

Instead of using the plug-in methods (e.g. KDE, BCOPS-RF), we solve each sub-problem by directly estimating the acceptance region through a constrained optimization problem. This has led to good generalization performance shown in our numerical studies. The Kernel Feature Selection mechanism can allow the proposed set-valued classifier to have better performance in terms of efficiency and outlier detection for high-dimensional data. Both theoretical and numerical studies demonstrate the usefulness of the proposed methods.

A possible future research direction is to develop methods without the need of test data. This may be done by learning an initial acceptance region for class $k$, and then improving it using data not from class $k$. Instead of the usage of test sub-sample in our proposed method, there are other works using auxiliary data \citep{hendrycks2018deep, neal2018open}, e.g. outlier exposure data or generating synthetic data. However, we may fail to guarantee the key metric ambiguity in set-valued classification if using the aforementioned auxiliary data, although we have improved detection rate and controlled class-specific accuracy. Another interesting problem is to generalize our methods to online learning \citep{kivinen2001online,lu2016large, zhang2019incremental}.

\bibliographystyle{asa}
\bibliography{LARfMCwAD}

\begin{thebibliography}{59}
\newcommand{\enquote}[1]{``#1''}
\expandafter\ifx\csname natexlab\endcsname\relax\def\natexlab#1{#1}\fi

\bibitem[{Allen(2013)}]{allen2013automatic}
Allen, G.~I. (2013), \enquote{Automatic feature selection via weighted kernels
  and regularization,} \textit{Journal of Computational and Graphical
  Statistics}, 22, 284--299.

\bibitem[{Balasubramanian et~al.(2014)Balasubramanian, Ho, and
  Vovk}]{balasubramanian2014conformal}
Balasubramanian, V., Ho, S.-S., and Vovk, V. (2014), \textit{Conformal
  prediction for reliable machine learning: theory, adaptations and
  applications}, Newnes.

\bibitem[{Bartlett and Wegkamp(2008)}]{bartlett2008classification}
Bartlett, P.~L. and Wegkamp, M.~H. (2008), \enquote{Classification with a
  reject option using a hinge loss,} \textit{Journal of Machine Learning
  Research}, 9, 1823--1840.

\bibitem[{Bendale and Boult(2015)}]{bendale2015towards}
Bendale, A. and Boult, T. (2015), \enquote{Towards open world recognition,} in
  \textit{Proceedings of the IEEE conference on computer vision and pattern
  recognition}, pp. 1893--1902.

\bibitem[{Blanchard et~al.(2010)Blanchard, Lee, and Scott}]{blanchard2010semi}
Blanchard, G., Lee, G., and Scott, C. (2010), \enquote{Semi-supervised novelty
  detection,} \textit{The Journal of Machine Learning Research}, 11,
  2973--3009.

\bibitem[{Breunig et~al.(2000)Breunig, Kriegel, Ng, and
  Sander}]{breunig2000lof}
Breunig, M.~M., Kriegel, H.-P., Ng, R.~T., and Sander, J. (2000), \enquote{LOF:
  identifying density-based local outliers,} in \textit{Proceedings of the 2000
  ACM SIGMOD international conference on Management of data}, pp. 93--104.

\bibitem[{Chen et~al.(2018)Chen, Zhang, Kosorok, and Liu}]{chen2018double}
Chen, J., Zhang, C., Kosorok, M.~R., and Liu, Y. (2018), \enquote{Double
  sparsity kernel learning with automatic variable selection and data
  extraction,} \textit{Statistics and its interface}, 11, 401.

\bibitem[{Chen et~al.(2017)Chen, Genovese, and Wasserman}]{chen2017density}
Chen, Y.-C., Genovese, C.~R., and Wasserman, L. (2017), \enquote{Density level
  sets: Asymptotics, inference, and visualization,} \textit{Journal of the
  American Statistical Association}, 112, 1684--1696.

\bibitem[{Chow(1970)}]{chow1970optimum}
Chow, C. (1970), \enquote{On optimum recognition error and reject tradeoff,}
  \textit{IEEE Transactions on information theory}, 16, 41--46.

\bibitem[{Denis and Hebiri(2015)}]{denis2015consistency}
Denis, C. and Hebiri, M. (2015), \enquote{Consistency of plug-in confidence
  sets for classification in semi-supervised learning,} \textit{arXiv preprint
  arXiv:1507.07235}.

\bibitem[{Denis and Hebiri(2017)}]{denis2017confidence}
--- (2017), \enquote{Confidence sets with expected sizes for multiclass
  classification,} \textit{The Journal of Machine Learning Research}, 18,
  3571--3598.

\bibitem[{du~Plessis et~al.(2014)du~Plessis, Niu, and
  Sugiyama}]{du2014analysis}
du~Plessis, M.~C., Niu, G., and Sugiyama, M. (2014), \enquote{Analysis of
  learning from positive and unlabeled data,} in \textit{Advances in neural
  information processing systems}, pp. 703--711.

\bibitem[{du~Plessis et~al.(2015)du~Plessis, Niu, and Sugiyama}]{du2015convex}
--- (2015), \enquote{Convex formulation for learning from positive and
  unlabeled data,} in \textit{International conference on machine learning},
  pp. 1386--1394.

\bibitem[{du~Plessis et~al.(2016)du~Plessis, Niu, and
  Sugiyama}]{christoffel2016class}
--- (2016), \enquote{Class-prior estimation for learning from positive and
  unlabeled data,} in \textit{Asian Conference on Machine Learning}, pp.
  221--236.

\bibitem[{Dümbgen et~al.(2008)Dümbgen, Igl, and Munk}]{D_mbgen_2008}
Dümbgen, L., Igl, B.-W., and Munk, A. (2008), \enquote{P-values for
  classification,} \textit{Electronic Journal of Statistics}, 2.

\bibitem[{Elkan and Noto(2008)}]{elkan2008learning}
Elkan, C. and Noto, K. (2008), \enquote{Learning classifiers from only positive
  and unlabeled data,} in \textit{Proceedings of the 14th ACM SIGKDD
  international conference on Knowledge discovery and data mining}, pp.
  213--220.

\bibitem[{Fan and Lv(2010)}]{fan2010selective}
Fan, J. and Lv, J. (2010), \enquote{A selective overview of variable selection
  in high dimensional feature space,} \textit{Statistica Sinica}, 20, 101.

\bibitem[{Fan and Peng(2004)}]{fan2004nonconcave}
Fan, J. and Peng, H. (2004), \enquote{Nonconcave penalized likelihood with a
  diverging number of parameters,} \textit{The annals of statistics}, 32,
  928--961.

\bibitem[{Guan and Tibshirani(2019)}]{guan2019prediction}
Guan, L. and Tibshirani, R. (2019), \enquote{Prediction and outlier detection
  in classification problems,} \textit{arXiv preprint arXiv:1905.04396}.

\bibitem[{Hanczar and Sebag(2014)}]{hanczar2014combination}
Hanczar, B. and Sebag, M. (2014), \enquote{Combination of one-class support
  vector machines for classification with reject option,} in \textit{Joint
  European Conference on Machine Learning and Knowledge Discovery in
  Databases}, Springer, pp. 547--562.

\bibitem[{Hechtlinger et~al.(2018)Hechtlinger, P{\'o}czos, and
  Wasserman}]{hechtlinger2018cautious}
Hechtlinger, Y., P{\'o}czos, B., and Wasserman, L. (2018), \enquote{Cautious
  deep learning,} \textit{arXiv preprint arXiv:1805.09460}.

\bibitem[{Hendrycks et~al.(2018)Hendrycks, Mazeika, and
  Dietterich}]{hendrycks2018deep}
Hendrycks, D., Mazeika, M., and Dietterich, T. (2018), \enquote{Deep anomaly
  detection with outlier exposure,} \textit{arXiv preprint arXiv:1812.04606}.

\bibitem[{Herbei and Wegkamp(2006)}]{herbei2006classification}
Herbei, R. and Wegkamp, M.~H. (2006), \enquote{Classification with reject
  option,} \textit{The Canadian Journal of Statistics/La Revue Canadienne de
  Statistique}, 709--721.

\bibitem[{Jumutc and Suykens(2013)}]{jumutc2013supervised}
Jumutc, V. and Suykens, J.~A. (2013), \enquote{Supervised novelty detection,}
  in \textit{2013 IEEE Symposium on Computational Intelligence and Data Mining
  (CIDM)}, IEEE, pp. 143--149.

\bibitem[{Kimeldorf and Wahba(1971)}]{kimeldorf1971some}
Kimeldorf, G. and Wahba, G. (1971), \enquote{Some results on Tchebycheffian
  spline functions,} \textit{Journal of mathematical analysis and
  applications}, 33, 82--95.

\bibitem[{Kivinen et~al.(2001)Kivinen, Smola, Williamson,
  et~al.}]{kivinen2001online}
Kivinen, J., Smola, A.~J., Williamson, R.~C., et~al. (2001), \enquote{Online
  Learning with Kernels.} in \textit{NIPS}, pp. 785--792.

\bibitem[{Lee et~al.(2012)Lee, Du, Sun, Hayes, and Liu}]{lee2012multiple}
Lee, W., Du, Y., Sun, W., Hayes, D.~N., and Liu, Y. (2012), \enquote{Multiple
  response regression for Gaussian mixture models with known labels,}
  \textit{Statistical Analysis and Data Mining: The ASA Data Science Journal},
  5, 493--508.

\bibitem[{Lei(2014)}]{lei2014classification}
Lei, J. (2014), \enquote{Classification with confidence,} \textit{Biometrika},
  101, 755--769.

\bibitem[{Lei et~al.(2015)Lei, Rinaldo, and Wasserman}]{lei2015conformal}
Lei, J., Rinaldo, A., and Wasserman, L. (2015), \enquote{A conformal prediction
  approach to explore functional data,} \textit{Annals of Mathematics and
  Artificial Intelligence}, 74, 29--43.

\bibitem[{Lei et~al.(2013)Lei, Robins, and Wasserman}]{lei2013distribution}
Lei, J., Robins, J., and Wasserman, L. (2013), \enquote{Distribution-free
  prediction sets,} \textit{Journal of the American Statistical Association},
  108, 278--287.

\bibitem[{Liu et~al.(2018)Liu, Garrepalli, Dietterich, Fern, and
  Hendrycks}]{liu2018open}
Liu, S., Garrepalli, R., Dietterich, T., Fern, A., and Hendrycks, D. (2018),
  \enquote{Open category detection with PAC guarantees,} in
  \textit{International Conference on Machine Learning}, PMLR, pp. 3169--3178.

\bibitem[{Lu et~al.(2016)Lu, Hoi, Wang, Zhao, and Liu}]{lu2016large}
Lu, J., Hoi, S.~C., Wang, J., Zhao, P., and Liu, Z.-Y. (2016), \enquote{Large
  scale online kernel learning,} \textit{Journal of Machine Learning Research},
  17, 1.

\bibitem[{Neal et~al.(2018)Neal, Olson, Fern, Wong, and Li}]{neal2018open}
Neal, L., Olson, M., Fern, X., Wong, W.-K., and Li, F. (2018), \enquote{Open
  set learning with counterfactual images,} in \textit{Proceedings of the
  European Conference on Computer Vision (ECCV)}, pp. 613--628.

\bibitem[{Ramaswamy et~al.(2015)Ramaswamy, Tewari, and
  Agarwal}]{ramaswamy2015consistent}
Ramaswamy, H.~G., Tewari, A., and Agarwal, S. (2015), \enquote{Consistent
  algorithms for multiclass classification with a reject option,} \textit{arXiv
  preprint arXiv:1505.04137}.

\bibitem[{Rigollet and Tong(2011)}]{rigollet2011neyman}
Rigollet, P. and Tong, X. (2011), \enquote{Neyman-pearson classification,
  convexity and stochastic constraints,} \textit{The Journal of Machine
  Learning Research}, 12, 2831--2855.

\bibitem[{Rosset and Zhu(2007)}]{rosset2007piecewise}
Rosset, S. and Zhu, J. (2007), \enquote{Piecewise linear regularized solution
  paths,} \textit{The Annals of Statistics}, 1012--1030.

\bibitem[{Ruff et~al.(2021)Ruff, Kauffmann, Vandermeulen, Montavon, Samek,
  Kloft, Dietterich, and M{\"u}ller}]{ruff2021unifying}
Ruff, L., Kauffmann, J.~R., Vandermeulen, R.~A., Montavon, G., Samek, W.,
  Kloft, M., Dietterich, T.~G., and M{\"u}ller, K.-R. (2021), \enquote{A
  unifying review of deep and shallow anomaly detection,} \textit{Proceedings
  of the IEEE}.

\bibitem[{Ruff et~al.(2018)Ruff, Vandermeulen, Goernitz, Deecke, Siddiqui,
  Binder, M{\"u}ller, and Kloft}]{ruff2018deep}
Ruff, L., Vandermeulen, R., Goernitz, N., Deecke, L., Siddiqui, S.~A., Binder,
  A., M{\"u}ller, E., and Kloft, M. (2018), \enquote{Deep one-class
  classification,} in \textit{International conference on machine learning},
  PMLR, pp. 4393--4402.

\bibitem[{Sadinle et~al.(2019)Sadinle, Lei, and Wasserman}]{sadinle2019least}
Sadinle, M., Lei, J., and Wasserman, L. (2019), \enquote{Least ambiguous
  set-valued classifiers with bounded error levels,} \textit{Journal of the
  American Statistical Association}, 114, 223--234.

\bibitem[{Sch{\"o}lkopf et~al.(2018)Sch{\"o}lkopf, Smola, and
  Bach}]{scholkopf2018learning}
Sch{\"o}lkopf, B., Smola, A.~J., and Bach, F. (2018), \textit{Learning with
  kernels: support vector machines, regularization, optimization, and beyond},
  the MIT Press, chap.~7, p. 209.

\bibitem[{Sch{\"o}lkopf et~al.(2000)Sch{\"o}lkopf, Williamson, Smola,
  Shawe-Taylor, and Platt}]{scholkopf2000support}
Sch{\"o}lkopf, B., Williamson, R.~C., Smola, A.~J., Shawe-Taylor, J., and
  Platt, J.~C. (2000), \enquote{Support vector method for novelty detection,}
  in \textit{Advances in neural information processing systems}, pp. 582--588.

\bibitem[{Scott and Nowak(2005)}]{scott2005neyman}
Scott, C. and Nowak, R. (2005), \enquote{A Neyman-Pearson approach to
  statistical learning,} \textit{IEEE Transactions on Information Theory}, 51,
  3806--3819.

\bibitem[{Shafer and Vovk(2008)}]{shafer2008tutorial}
Shafer, G. and Vovk, V. (2008), \enquote{A tutorial on conformal prediction,}
  \textit{Journal of Machine Learning Research}, 9, 371--421.

\bibitem[{Shawe-Taylor and Cristianini(2004)}]{shawe2004kernel}
Shawe-Taylor, J. and Cristianini, N. (2004), \textit{Kernel methods for pattern
  analysis}, Cambridge university press.

\bibitem[{Shilton et~al.(2020)Shilton, Rajasegarar, and
  Palaniswami}]{shilton2020multiclass}
Shilton, A., Rajasegarar, S., and Palaniswami, M. (2020), \enquote{Multiclass
  Anomaly Detector: the CS++ Support Vector Machine.} \textit{J. Mach. Learn.
  Res.}, 21, 213--1.

\bibitem[{Silverman(2018)}]{silverman2018density}
Silverman, B.~W. (2018), \textit{Density estimation for statistics and data
  analysis}, Routledge.

\bibitem[{Steinwart and Christmann(2008)}]{steinwart2008support}
Steinwart, I. and Christmann, A. (2008), \textit{Support vector machines},
  Springer Science \& Business Media.

\bibitem[{Steinwart et~al.(2005)Steinwart, Hush, and
  Scovel}]{steinwart2005classification}
Steinwart, I., Hush, D., and Scovel, C. (2005), \enquote{A classification
  framework for anomaly detection,} \textit{Journal of Machine Learning
  Research}, 6, 211--232.

\bibitem[{Tibshirani(1996)}]{tibshirani1996regression}
Tibshirani, R. (1996), \enquote{Regression shrinkage and selection via the
  lasso,} \textit{Journal of the Royal Statistical Society: Series B
  (Methodological)}, 58, 267--288.

\bibitem[{Vovk et~al.(2005)Vovk, Gammerman, and Shafer}]{vovk2005algorithmic}
Vovk, V., Gammerman, A., and Shafer, G. (2005), \textit{Algorithmic learning in
  a random world}, Springer Science \& Business Media.

\bibitem[{Wang and Qiao(2018)}]{wang2018learning}
Wang, W. and Qiao, X. (2018), \enquote{Learning confidence sets using support
  vector machines,} in \textit{Advances in Neural Information Processing
  Systems}, pp. 4929--4938.

\bibitem[{Wu et~al.(2010)Wu, Zhang, and Liu}]{wu2010robust}
Wu, Y., Zhang, H.~H., and Liu, Y. (2010), \enquote{Robust model-free multiclass
  probability estimation,} \textit{Journal of the American Statistical
  Association}, 105, 424--436.

\bibitem[{Yang et~al.(2021)Yang, Zhou, Li, and Liu}]{yang2021oodsurvey}
Yang, J., Zhou, K., Li, Y., and Liu, Z. (2021), \enquote{Generalized
  Out-of-Distribution Detection: A Survey,} \textit{arXiv preprint
  arXiv:2110.11334}.

\bibitem[{Zhang et~al.(2013)Zhang, Liu, and Wu}]{zhang2013effect}
Zhang, C., Liu, Y., and Wu, Z. (2013), \enquote{On the effect and remedies of
  shrinkage on classification probability estimation,} \textit{The American
  Statistician}, 67, 134--142.

\bibitem[{Zhang et~al.(2018)Zhang, Wang, and Qiao}]{zhang2018reject}
Zhang, C., Wang, W., and Qiao, X. (2018), \enquote{On reject and refine options
  in multicategory classification,} \textit{Journal of the American Statistical
  Association}, 113, 730--745.

\bibitem[{Zhang(2010)}]{zhang2010nearly}
Zhang, C.-H. (2010), \enquote{Nearly unbiased variable selection under minimax
  concave penalty,} \textit{The Annals of statistics}, 38, 894--942.

\bibitem[{Zhang and Liao(2019)}]{zhang2019incremental}
Zhang, X. and Liao, S. (2019), \enquote{Incremental randomized sketching for
  online kernel learning,} in \textit{International Conference on Machine
  Learning}, PMLR, pp. 7394--7403.

\bibitem[{Zou and Hastie(2005)}]{zou2005regularization}
Zou, H. and Hastie, T. (2005), \enquote{Regularization and variable selection
  via the elastic net,} \textit{Journal of the royal statistical society:
  series B (statistical methodology)}, 67, 301--320.

\bibitem[{Zou and Li(2008)}]{zou2008one}
Zou, H. and Li, R. (2008), \enquote{One-step sparse estimates in nonconcave
  penalized likelihood models,} \textit{Annals of statistics}, 36, 1509.

\end{thebibliography}


\clearpage

\appendix

\section{Numerical study details}
\subsection{Algorithm outline}\label{alg:outline}

\begin{algorithm}
    \caption{Optimization with Weighted Kernel Feature Selection}
    \label{alg1}
\begin{algorithmic}[1]
\Require Initialize $(\bm\alpha^{(0)}, \rho_k^{(0)})=\bm0, \bm d^{(0)}=\mathbf 1$
\While{$(\bm\alpha^{(t)}, \rho_k^{(t)}), \bm d^{(t)}$ not convergent}
    \State $(\bm\alpha^{(t)},\rho_k^{(t)})=$ minimizer of problem (\ref{eq:KFSprob2}) when fixing $\bm d=\bm d^{(t)}$
    \While{$\bm d^{\mbox{cv}}$ not convergent}
  
         \State $\bm d^{\mbox{cv}}=$ minimizer of problem (\ref{eq:KFSprob3}) when fixing $(\bm\alpha,\rho_k)=(\bm\alpha^{(t-1)},\rho_k^{(t-1)})$ 
         
         \parState{Compute a descent direction $\Delta \bm d=\bm d^{\mbox{cv}}-\bm d^{(t-1)}$. Conduct a line search to find $\nu$ such that $\bm d^{(t-1)}+\nu\Delta \bm d$ decreasing the objective function in problem (\ref{eq:KFSprob2}) when fixing $(\bm\alpha,\rho_k)=(\bm\alpha^{(t-1)},\rho_k^{(t-1)})$}
                  
         \State $\bm d^{(t-1)}=\bm d^{(t-1)}+\nu\Delta \bm d$ 
         \EndWhile
         
    \State $\bm d^{(t)}=\bm d^{\mbox{cv}}$
    
\EndWhile
\State \Return ($\bm\alpha^{(t)}, \rho_k^{(t)}, \bm d^{(t)}$), and $\hat f_k(\bm d \circ \cdot)$
\end{algorithmic}
\end{algorithm}
\subsection{Loss functions}\label{loss:func}
Proposition \ref{prop:differen} shows that we can obtain a local minimum of the objective in each iteration using Algorithm \ref{alg1} under a certain loss function.

\begin{proposition}[Proposition 1 in \citet{allen2013automatic}]
\label{prop:differen}
If the convex loss function in (\ref{eq:KFSprob2}) and (\ref{eq:KFSprob3}) is continuously differentiable with respect to $(\bm\alpha,\rho_k)$, the kernel function is convex or concave and is continuously differentiable with respect to $\bm d$, then the solution obtained from Algorithm \ref{alg1}  converges to a local minimizer.
\end{proposition}

The hinge loss is not continuously differentiable as required by Proposition \ref{prop:differen}. One may substitute the hinge loss with a differentiable loss function such as the logistic loss, the squared hinge loss, or the Huberized squared hinge loss \citep{rosset2007piecewise}:
\begin{equation}\label{eq:hubhinge}
\ell(u)=\left\{\begin{array}{ll}
  1-u, &\quad u\leq 1-\delta\\
  \frac{(1-u+\delta)^2}{4\delta}, &\quad 1-\delta<u\leq1+\delta\\
  0, &\quad u>1+\delta
\end{array}\right.
\end{equation}

The parameter $\delta$ here is specified by the user. From Figure \ref{fig:lossfun} we see that the Huberized squared hinge loss approximates the hinge loss with small $\delta$. In this paper we use the Huberized hinge loss with $\delta =0.1$ to replace the hinge loss for kernel learning in order to improve the performance.

\begin{figure}[!ht]
\begin{center}
\adjustimage{scale = 0.55}{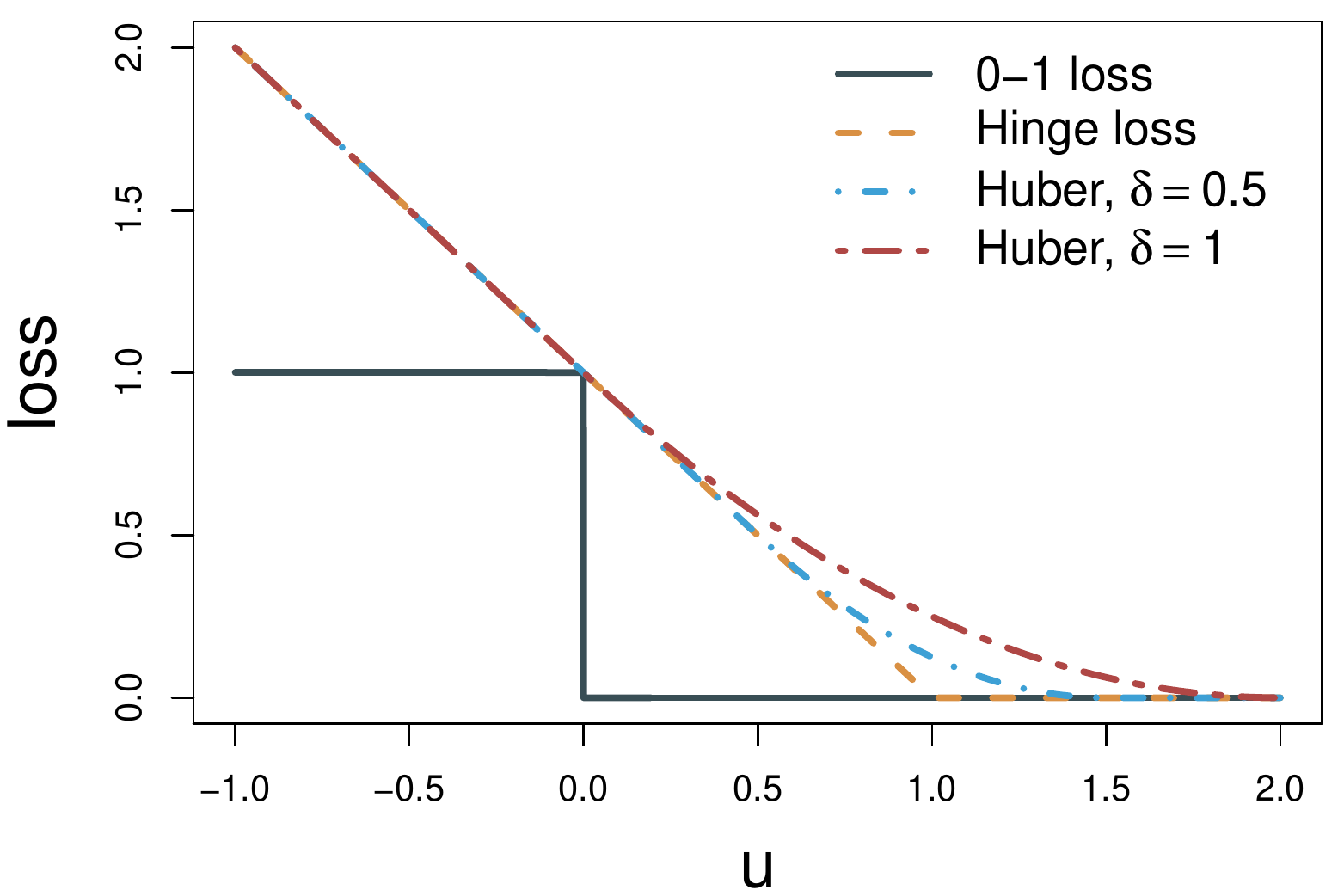}
\vspace{-1em}
\caption{0-1 loss, hinge loss, and two Huberized loss functions.}
\label{fig:lossfun}
\end{center}
\vspace{-1em}
\end{figure}

\subsection{Implementation details}\label{appedx:implem}
To choose the tuning parameters, the candidate hyper-parameters $C_1, C_2$ in GPSKFS is searched from grid $\{1, 2, 3\}$ and $10^{\wedge}\{\pm1, \pm0.75, \pm0.5, \pm0.25, 0\}$, respectively. The hyper-parameter $C$ in GPS is searched from the grid $10^{\wedge}\{\pm2, \pm1.5, \pm1, \pm0.5, 0\}$. For the $\sigma$ parameter in the Gaussian kernel $\exp(-\|\bm x-\bm x^\prime\|^2/\sigma^2)$, we choose it from the $\{25, 37.5, 50, 62.5, 75\}$-th percentiles of all the pairwise weighted Euclidean distances over the training sample $\|\bm d \circ(\bm x-\bm x^\prime)\|_2$, where $\bm d$ is the current estimated weight vector which can itself evolve in the iterations. For the KDE method, the bandwidth is searched from a grid $\{\hat\sigma_{(1)}, \hat\sigma_{(1)} + \frac{\hat\sigma_{(p)}-\hat\sigma_{(1)}}{p-1}, \ldots, \hat\sigma_{(p)}\}\times (\frac{4}{(p+2)n})^{1/(p+4)}$ based on Silverman's  rule-of-thumb bandwidth estimator \citep{silverman2018density}, where $\hat\sigma_{(1)}$ and $\hat\sigma_{(p)}$ are the minimum and maximum standard deviation among all columns of data. We search parameter $\sigma$ in the Gaussian kernel for both OCSVM and BSVM in the same way as in GPS. For BSVM, the parameter $C$ is searched from the same grid as the one for $C_1$ in GPS. The parameter $\nu$ in OCSVM is the upper bound of the overall proportion of points outside of any acceptance region, and hence is set as $\gamma$, which is its class-specific counterpart in our paper. For BCOPS-RF, the maximum depth of the tree is searched from $\{10, 20, \ldots, 90, 100\}$. Minimum samples to split an internal node, minimum samples at a leaf node, and the number of trees are searched from $\{2, 5, 10\}, \{2, 4, 6\}$, and $\{50, 150, 200\}$, respectively. All parameters are determined such that the cardinality of the prediction set is minimized on the calibration data.

To conduct a fair comparison, we use the conformal inference framework for all classifiers. In particular, we make use of the split-conformal method \citep{lei2013distribution, lei2014classification, lei2015conformal}, which does not incur much extra computation burden. For example, for the GPSKFS method, we take $\hat f_k(\bm d\circ\bm x)$ as the conformal score function. Given any class $k$, we randomly split the original training sample from class $k$ with the test data (or its subset) into the training part and the calibration part. We use the first part to train the classifier and the second part to do the calibration and select tuning parameters. The threshold $\hat\tau_k$ is determined based on pre-specified significance level $\gamma$, where $\hat\tau_k$ is taken as $(\gamma\times 100)$-th percentile of the scores among the calibration data. Finally, the prediction set for a given $\bm x$ is $\{k\in[K]:  \hat f_k(\bm d\circ\bm x)\geq\hat\tau_k\}$. Likewise, an outlier is detected when $\hat f_k(\bm d\circ\bm x)<\hat\tau_k$ for all $k\in[K]$. This improvement is applied to all competing methods with the conformal score function chosen appropriately for each method.

\begin{proposition}
Let $\widehat \phi$ be a prediction set induced by $ \hat f_k(\bm d\circ\cdot), k\in[K]$ learned from a model under sample-splitting approach and $(\bm X, Y)$ be an independent new instance, then $$\mathbb{P}[Y\not\in\widehat\phi(\bm X)\mid Y = k]\leq \gamma, \ k\in[K]$$ for any distribution.
\end{proposition}

The proof of this proposition can be found in \citet{lei2015conformal}. It theoretically guarantees accuracy as long as class-conditional distributions are preserved between training and test data, which is fulfilled by Assumption \ref{assum:1}.

\subsection{Simulation data}
\label{app:simdt}

\begin{figure}[!h]
  \centering
  \includegraphics[width=1\linewidth]{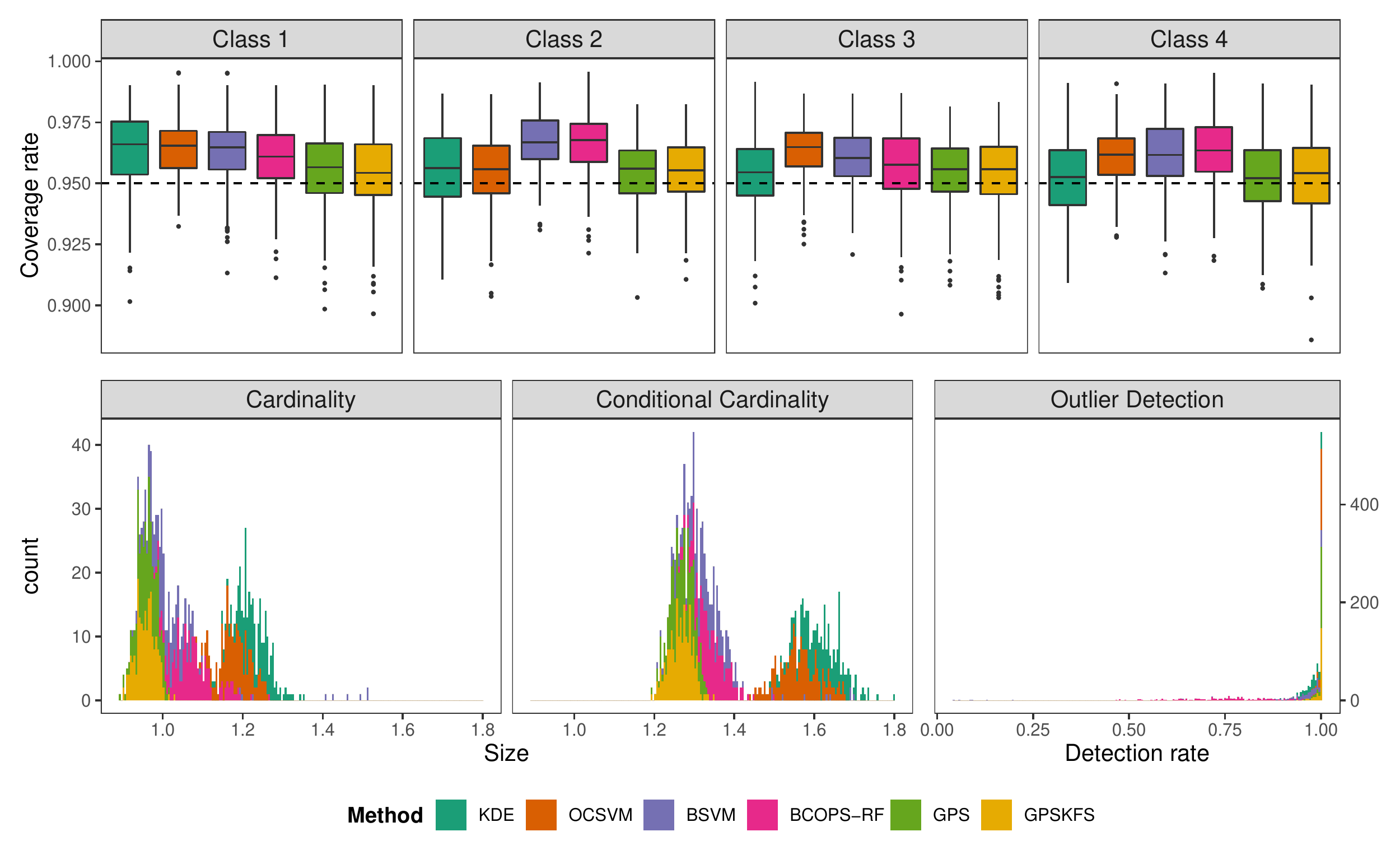}
  \vskip -0.15in
  \caption{Distributions of performance metrics for Example 1: box plots of the coverage rate in the top panel, and histograms of the cardinality, the conditional cardinality, and the outlier detection rate in the bottom panel.}
  \label{fig:cvgcardod1}
\end{figure}

We also show distributions of the coverage rate in our simulation using box plots in the top panel of \cref{fig:cvgcardod1}. The bottom panel of \cref{fig:cvgcardod1} shows the histograms of these metrics in our simulation, which further confirm the high efficiency and high outlier detection of our proposed methods.

\begin{figure*}[!thb]
  \centering
  \includegraphics[width=1\linewidth]{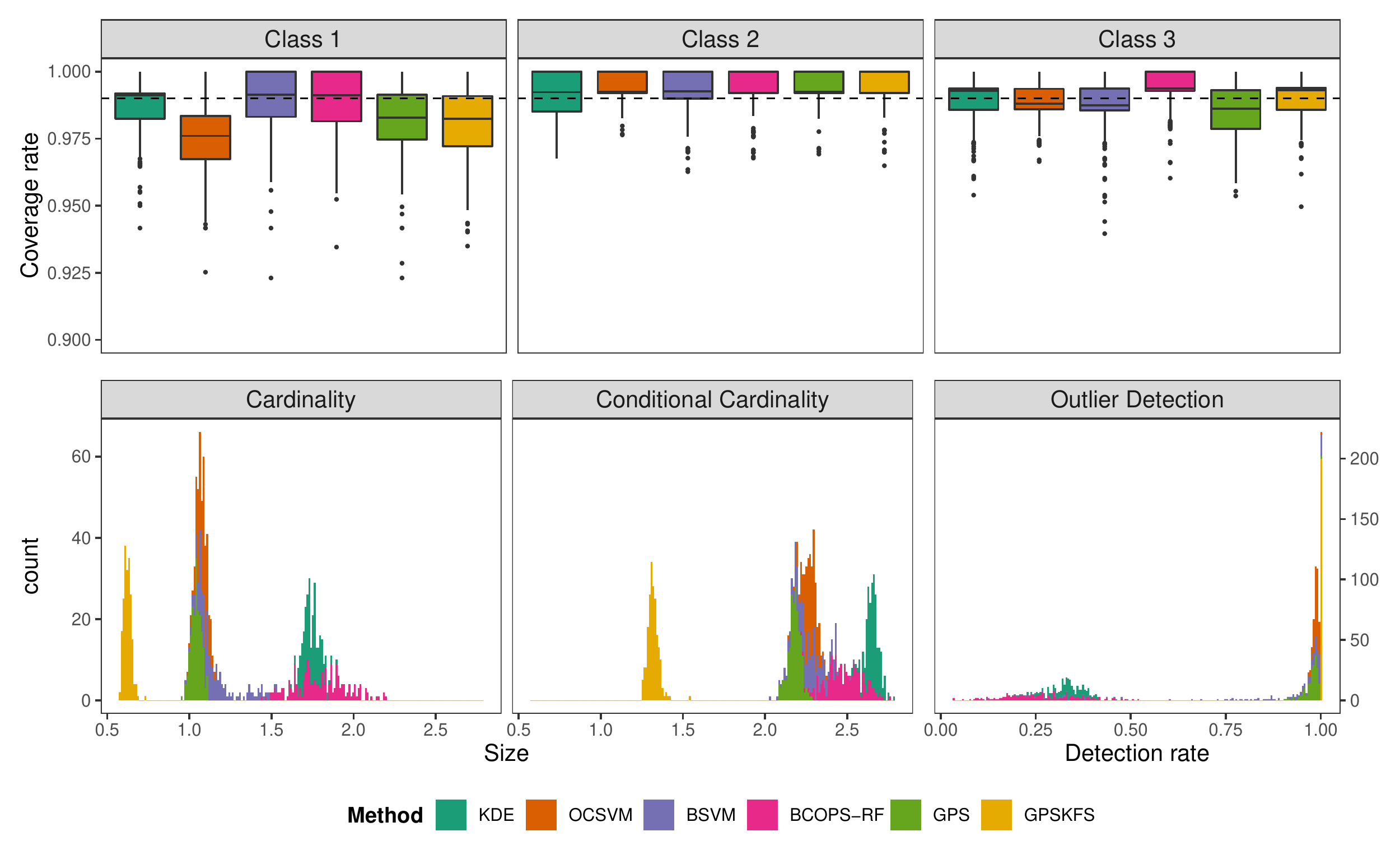}
  \vskip -0.05in
  \caption{Distributions of performance metrics for Example 2: box plots of the coverage rate in the top panel, and histograms of cardinality, conditional cardinality, and outlier detection rate in the bottom panel.}
  \label{fig:cvgcardod2}
\end{figure*}

\cref{fig:cvgcardod2} gives more clear detail on performances, in which box plots of the coverage rates are displayed. From the bottom panel, we can see that the detection rate of plug-in methods (KDE and BCOPS-RF) has a scattered distribution. The prediction set cardinalities for GPSKFS are the smallest and the detection rate concentrates at 100\%. This observation demonstrates the effectiveness of feature selection.

\clearpage

\section{Proofs}
\label{sec:proof}

\textbf{Proof of \cref{prop:differen}}: 
\begin{proof}
The loss function is always bounded below; to prove that it converges to a stationary point, it suffices to prove the Algorithm decreases in each step. Denote the original objective function in Problem (\ref{eq:KFSprob2}) as $\Psi(\bm\alpha, \rho_k, \bm d)$. It is easy to conclude $\Psi(\bm\alpha^{(t)}, \rho_k^{(t)}, \bm d^{(t-1)})\leq \Psi(\bm\alpha^{(t-1)}, \rho_k^{(t-1)}, \bm d^{(t-1)})$ because updating for $(\bm\alpha, \rho_k)$ when fixing $\bm d^{(t-1)}$ is a convex optimization problem. Thus, it suffices to verify $\Psi(\bm\alpha^{(t-1)}, \rho_k^{(t-1)}, \bm d^{(t)})\leq \Psi(\bm\alpha^{(t-1)}, \rho_k^{(t-1)}, \bm d^{(t-1)})$ when fixing $(\bm\alpha^{(t-1)}, \rho_k^{(t-1)})$ and updating for $\bm d$. We only focus on the case where $\frac{\partial\Psi}{\partial\bm d}\neq\bm 0$ at $(\bm\alpha^{(t-1)}, \rho_k^{(t-1)}, \bm d^{(t-1)})$; otherwise we already arrive at a stationary point.

First of all, define $$\mathbf{G}(\bm d)=\left[g_{i,j}(\bm d)\right]_{i,j}
=\begin{bmatrix}
   \mathbf{K}_{\bm d} &&&\\
    & \bm e_1^\top\bm d &  &  \\
    &  & \ddots &  \\
    &  &  & \bm e_p^\top \bm d
 \end{bmatrix} ~~\mbox{and}~~\widetilde{\bm\alpha}=\begin{bmatrix}
   \frac{1}{\sqrt{2}}\bm\alpha\\
   \sqrt{C_2}\bm1_p
 \end{bmatrix},$$
 where $\mathbf{K}_{\bm d}$ is a $(n+m)\times(n+m)$ kernel matrix, $\bm e_l$ is a column vector with $l$-th element 1 but 0 elsewhere, and $\bm 1_p$ is a $p$-dimensional column vector with all 1's. Given the above notations, the scalar $\mathbf{K}_{\bm d}[j, :]\bm\alpha$ for some $j$ can be written as $\sum\limits_i\tilde{\beta}_ig_{i,j}(\bm d)$ for some $\tilde{\beta}_i$'s. Then when fixing $(\bm\alpha^{(t-1)}, \rho_k^{(t-1)})$, we write the original objective function as a function of $\bm d$ only:
$$\Psi(\bm d)=C_1\sum\limits_j\ell(\rho_k^{(t-1)}-\sum\limits_i\tilde{\beta}_ig_{i,j}(\bm d))+\sum\limits_i\sum\limits_j\tilde\alpha_i\tilde\alpha_jg_{i,j}(\bm d).$$ Since $C_1>0$ and $\ell(\cdot)$ is convex, the objective function is still convex with respect to $g_{i,j}(\bm d)$. Without loss of generality and for the simplicity of notation, we can consider minimizing an objective function $\Psi(\bm d)=h\left(g(\bm d)\right)$, where $h(\cdot)$ is a continuously differentiable and convex function, and $g(\bm d)$ is continuously differentiable and convex or concave with respect to $\bm d$ (because of the assumption for kernel functions and the property of $\bm e_l^\top\bm d$). Moreover, denote $\widetilde{\Psi}_{\bm d^{(t-1)}}(\bm d)=h(g(\bm d^{(t-1)})+\nabla g(\bm d^{(t-1)})^\top (\bm d-\bm d^{(t-1)}))$ as the approximated objective function where we linearize the kernel function at $\bm d^{(t-1)}$ to obtain a convex optimization Problem (10). For this sub-optimization problem, we always have $\widetilde{\Psi}_{\bm d^{(t-1)}}(\bm d^{(t)})\leq \widetilde{\Psi}_{\bm d^{(t-1)}}(\bm d^{(t-1)})$.

Now we only need to verify $\Psi(\bm d^{(t)})\leq \Psi(\bm d^{(t-1)})$ for those cases \citep{allen2013automatic}: (1)  $h(\cdot)$ is deceasing or increasing when $g(\cdot)$ is convex, and (2) $h(\cdot)$ is deceasing or increasing when $g(\cdot)$ is concave. 

When $g(\cdot)$ is convex, we have $g(\bm d)\geq g(\bm d^{(t-1)})+\nabla g(\bm d^{(t-1)})^\top(\bm d -\bm d^{(t-1)})$. If $h(\cdot)$ is decreasing, then we have $$\begin{aligned}
h(g(\bm d))&\leq h(g(\bm d^{(t-1)})+\nabla g(\bm d^{(t-1)})^\top(\bm d -\bm d^{(t-1)}))\\
\Rightarrow h(g(\bm d^{(t)}))&\leq h(g(\bm d^{(t-1)})+\nabla g(\bm d^{(t-1)})^\top(\bm d^{(t)} -\bm d^{(t-1)}))\\
\Rightarrow \Psi(\bm d^{(t)})&\leq\widetilde{\Psi}_{\bm d^{(t-1)}}(\bm d^{(t)})\leq \widetilde{\Psi}_{\bm d^{(t-1)}}(\bm d^{(t-1)})=\Psi(\bm d^{(t-1)}),
\end{aligned}$$ which implies the original objective function decreases at this step although the solution $\bm d^{(t)}$ is obtained by solving Problem (\ref{eq:KFSprob3}).

On the other hand, for any $0\leq a\leq1$, the convexity of $g$ yields $$g(a\bm d+(1-a)\bm d^{(t-1)})\leq ag(\bm d)+(1-a)g(\bm d^{(t-1)}).$$ If $h(\cdot)$ is increasing and convex, then we have $$\begin{aligned}
 \Psi(a\bm d+(1-a)\bm d^{(t-1)})= h(g(a\bm d+(1-a)\bm d^{(t-1)}))&\leq h(ag(\bm d)+(1-a)g(\bm d^{(t-1)})) \\
&\leq ah(g(\bm d))+(1-a)h(g(\bm d^{(t-1)}))\\
&=a\Psi(\bm d)+(1-a)\Psi(\bm d^{(t-1)}),
\end{aligned}$$
which implies $\Psi(\bm d)$ is convex at the neighborhood of $\bm d^{(t-1)}$, say $N(\bm d^{(t-1)})$.

Since $\Psi(\bm d)$ is locally convex in $N(\bm d^{(t-1)})$, we can decrease it by taking a proper direction. So we take $\Delta \bm d=\bm d^{\mbox{cv}}-\bm d^{(t-1)}$ as a descent direction with a proper step size $s$ by the line search to decrease $\Psi(\bm d)$, where $s$ is to make sure $\Psi(\bm d)$ is decreased in the feasible region.

For the other two cases where $g(\cdot)$ is concave, similarly, we can verify $\Psi(\bm d)$ also decreases when fixing $(\bm\alpha^{(t-1)},\rho_k^{(t-1)})$. Therefore, the solution obtained from the algorithm converges to a local minimizer.
\end{proof}

\

First of all, we need to introduce below lemma on the boundness of $\rho$ and $g$.

\begin{lemma}
 Let $f(\cdot)=g(\cdot)-\rho\in \mathcal{F}_{s,s^\prime} (s, s^\prime\geq0)$, where $g=\sum_{i=1}^{n_1+m}\alpha_i\Phi(\bm x_i)$ belongs to the Gaussian kernel RKHS. We have, $\rho\leq\sqrt{2}s+2$ and $\|g\|\leq\sqrt{2}s+2$.
\end{lemma}
\begin{proof}
Under the Gaussian kernel, the distance from the hyper-plane to the origin is $\frac{\rho}{\|\bm w\|}\leq 1$. Together with the hypothesis space complexity $\frac{1}{2}\|g\|^2-\rho\leq s^2$, we have $\rho\leq \sqrt{2s^2+1}+1\leq \sqrt{2}s+2$ and hence $\|g\|\leq \sqrt{2(s^2+\sqrt{2}s+2)}\leq \sqrt{2}s+2$.
\end{proof}

\textbf{Proof of \cref{thm:type1bnd}}: \begin{proof}
For simplicity, denote $\mathbb{E}_Q\left[\ell(f(\bm X))\mid Y=1\right]=E_+\left[\ell(f(\bm X))\right]$ and hence $\mathbb{P}_Q[f(\bm X)<0\mid Y=1]\leq E_+[\ell(f(\bm X))]$. Define $\psi(S)=\sup_{f\in \mathcal{F}^+_{s, s^\prime}(\gamma)}E_+[\ell(f(\bm X))]-\frac{1}{n_1}\sum_{\bm x_i\in S}\ell(f(\bm x_i))$ and let $S^\prime$ be another sample from $\mathbb{P}_Q[\cdot\mid Y=1]$ but only different from $S$ on one observation $(\bm x^\prime, 1)$. Thus
\begin{equation*}
\begin{aligned}
    \left\|\psi(S)-\psi(S^\prime)\right\|&=\left\|\left(\sup\limits_{f\in \mathcal{F}^+_{s, s^\prime}(\gamma)}E_+[\ell(f(\bm X))]-\frac{1}{n_1}\sum\limits_{\bm x_i\in S}\ell(f(\bm x_i))\right)\right. \\
    &\quad - \left.\left(\sup\limits_{f\in \mathcal{F}^+_{s, s^\prime}(\gamma)}E_+[\ell(f(\bm X))]-\frac{1}{n_1}\sum\limits_{\bm x^\prime_i\in S^\prime}\ell(f(\bm x^\prime_i))\right)\right\|\\
    &\leq\frac{1}{n_1}\sup\limits_{f\in \mathcal{F}^+_{s, s^\prime}(\gamma)}\left\|\ell(f(\bm x))-\ell(f(\bm x^\prime))\right\|\\
    &\leq\frac{c}{n_1}\sup\limits_{\mathcal{F}^+_{s, s^\prime}(\gamma)}\left\|g(\bm x)-g(\bm x^\prime)\right\|\\
    &\leq\frac{2c}{n_1}\sup\limits_{\mathcal{F}^+_{s, s^\prime}(\gamma)}\left\|\langle g, K_{\bm d}(\bm x, \cdot)\rangle\right\|\\
    &\leq\frac{2(\sqrt{2}s+2)c\kappa}{n_1}.
\end{aligned}
\end{equation*}
Together with McDiarmid inequality, with probability $1-\zeta$, we have
\begin{equation*}
\psi(S)\leq \mathop{E_+}\limits_{S}[\psi(S)]+(\sqrt{2}s+2)c\kappa\sqrt{\frac{2\log{\frac{1}{\zeta}}}{n_1}},
\end{equation*} and hence
\begin{equation*}
  E_+[\ell(f(\bm X))]\leq\frac{1}{n_1}\sum\limits_{i=1}^{n_1}\ell(f(\bm x_i))+\mathop{E_+}\limits_{S}[\psi(S)]+(\sqrt{2}s+2)c\kappa\sqrt{\frac{2\log{\frac{1}{\zeta}}}{n_1}},
\end{equation*}
where \begin{equation*}
\begin{aligned}
    \mathop{E_+}\limits_S[\psi(S)]&=  \mathop{E_+}\limits_S\left[\sup\limits_{f\in \mathcal{F}^+_{s, s^\prime}(\gamma)}E_+[\ell(f(\bm X))]-\frac{1}{n_1}\sum\limits_{\bm x_i\in S}\ell(f(\bm x_i))\right]\\
    &=\mathop{E_+}\limits_{S}\left[\sup\limits_{f\in \mathcal{F}^+_{s, s^\prime}(\gamma)}\mathop{E_+}\limits_{S^\prime}\left[\frac{1}{n_1}\sum\limits_{\bm x_i^\prime\in S^\prime}\ell(f(\bm x_i^\prime))\right]-\frac{1}{n_1}\sum\limits_{\bm x_i\in S}\ell(f(\bm x_i))\right]\\
    &\leq\mathop{E_+}\limits_{S}\mathop{E_+}\limits_{S^\prime}\left[\sup\limits_{f\in \mathcal{F}^+_{s, s^\prime}(\gamma)}\frac{1}{n_1}\sum\limits_{\bm x_i^\prime\in S^\prime}\ell(f(\bm x_i^\prime))-\frac{1}{n_1}\sum\limits_{\bm x_i\in S}\ell(f(\bm x_i))\right]\\
    &=\mathop{E_+}\limits_{S}\mathop{E_+}\limits_{S^\prime}\mathop{\mathbb{E}}\limits_{\bm\sigma}\sup\limits_{f\in \mathcal{F}^+_{s, s^\prime}(\gamma)}\frac{1}{n_1}\sum\limits_{i=1}^{n_1}\sigma_i\left[\ell(f(\bm x_i^\prime))-\ell(f(\bm x_i))\right]\\
    &\leq\mathop{E_+}\limits_{S}\mathop{E_+}\limits_{S^\prime}\mathop{\mathbb{E}}\limits_{\bm\sigma}\sup\limits_{f\in \mathcal{F}^+_{s, s^\prime}(\gamma)}\frac{1}{n_1}\sum\limits_{i=1}^{n_1}\sigma_i\ell(f(\bm x_i^\prime))+\mathop{E_+}\limits_{S}\mathop{E_+}\limits_{S^\prime}\mathop{\mathbb{E}}\limits_{\bm\sigma}\sup\limits_{f\in \mathcal{F}^+_{s, s^\prime}(\gamma)}\frac{1}{n_1}\sum\limits_{i=1}^{n_1}-\sigma_i\ell(f(\bm x_i^\prime))\\
    &=2\mathop{E_+}\limits_{S}\mathop{\mathbb{E}}\limits_{\bm\sigma}\sup\limits_{f\in \mathcal{F}^+_{s, s^\prime}(\gamma)}\frac{1}{n_1}\sum\limits_{\bm x_i\in S}\sigma_i\ell(f(\bm x_i))\\
    &=2\mathfrak{R}_{n_1}(\ell\circ \mathcal{F}^+_{s, s^\prime}(\gamma))
\end{aligned}
\end{equation*}
Applied again with McDiarmid inequality, with probability $1-\zeta$, we have $$\mathfrak{{R}}_{n_1}(\ell\circ \mathcal{F}^+_{s, s^\prime}(\gamma))\leq\mathfrak{\widehat{R}}_{n_1}(\ell\circ \mathcal{F}^+_{s, s^\prime}(\gamma))+(\sqrt{2}s+2)c\kappa\sqrt{\frac{2\log{\frac{1}{\zeta}}}{n_1}}.$$

According to Talagrand's lemma, $$\mathfrak{\widehat{R}}_{n_1}(\ell\circ \mathcal{F}^+_{s, s^\prime}(\gamma))\leq c\cdot\mathfrak{\widehat{R}}_{n_1}(\mathcal{F}^+_{s, s^\prime}(\gamma)).$$
Since $$\begin{aligned}
  \mathfrak{\widehat{R}}_{n_1}( \mathcal{F}^+_{s, s^\prime}(\gamma))&=\mathop{\mathbb{E}}\limits_{\bm\sigma}\sup\limits_{f\in \mathcal{F}^+_{s, s^\prime}(\gamma)}\frac{1}{n_1}\sum\limits_{\bm x_i\in S}\sigma_if(\bm x_i)\\
  &\leq\mathop{\mathbb{E}}\limits_{\bm\sigma}\sup\limits_{\mathcal{F}^+_{s, s^\prime}(\gamma)}\frac{1}{n_1}\sum\limits_{\bm x_i\in S}\sigma_i\langle g, K_{\bm d}(\bm x_i, \cdot)\rangle+\mathop{\mathbb{E}}\limits_{\bm\sigma}\sup\limits_{\mathcal{F}^+_{s, s^\prime}(\gamma)}\frac{1}{n_1}\sum\limits_{\bm x_i\in S}-\sigma_i\rho\\
  &\leq\frac{(\sqrt{2}s+2)}{n_1}\mathop{\mathbb{E}}\limits_{\bm\sigma}\left|\sum\limits_{\bm x_i\in S}\sigma_i\sqrt{K_{\bm d}(\bm x_i, \bm x_i)}\right|\\
  &\leq\frac{(\sqrt{2}s+2)}{n_1}\left[\mathop{\mathbb{E}}\limits_{\bm\sigma}\left(\sum\limits_{\bm x_i\in S}\sigma_i\sqrt{K_{\bm d}(\bm x_i, \bm x_i)}\right)^2\right]^{\frac{1}{2}}\\
  &\leq\frac{(\sqrt{2}s+2)}{n_1}\left(n_1\kappa^2\right)^\frac{1}{2}=\frac{(\sqrt{2}s+2)\kappa}{\sqrt{n_1}},
\end{aligned}$$ combining above results, with probability $1-2\zeta$, we have
\begin{equation*}
  \mathbb{E}_Q\left[\ell(f(\bm X))\mid Y=1\right]\leq\frac{1}{n_1}\sum\limits_{i=1}^{n_1}\ell(f(\bm x_i))+\frac{2(\sqrt{2}s+2)c\kappa}{\sqrt{n_1}}+3(\sqrt{2}s+2)c\kappa\sqrt{\frac{2\log{\frac{1}{\zeta}}}{n_1}}.
\end{equation*}
\end{proof}

Before proving \cref{thm:estmationerror}, we prove the below proposition.

\begin{proposition}
  Let $\nu(\gamma)=\inf_{f\in\mathcal{F}^+_{s, s^\prime}(\gamma)}\mathcal{R}_\ell(f)$, then $\nu$ is a non-increasing convex function on $[0, 1]$.
\end{proposition}

\begin{proof}
$\nu$ is non-increasing because of the definition of infimum. We now focus on the convexity. $\mathcal{F}^+_{s, s^\prime}(\gamma)$ is compact due to continuouity and boundedness of $\ell$ and $f$, therefore, there exits a $f^\gamma\in\mathcal{F}^+_{s, s^\prime}(\gamma)$ such that $\nu(\gamma)=\mathcal{R}_\ell(f^\gamma)$.

Let $\nu(\gamma_1)=\mathcal{R}_\ell(f^{\gamma_1}), \nu(\gamma_2)=\mathcal{R}_\ell(f^{\gamma_2})$ and define
$\gamma_\theta=\theta\gamma_1+(1-\theta)\gamma_2, f_\theta=\theta f^{\gamma_1}+(1-\theta)f^{\gamma_2}$ for any $\theta\in(0, 1)$, since $\ell$ is convex, then we have
$$E_+[\ell\circ f_\theta]\leq \theta E_+[\ell\circ f^{\gamma_1}]+(1-\theta) E_+[\ell\circ f^{\gamma_2}]\leq\theta {\gamma_1}+(1-\theta) {\gamma_2}=\gamma_{\theta}$$ and hence $$\nu(\theta\gamma_1+(1-\theta)\gamma_2)=\nu(\gamma_\theta)\leq\mathcal{R}_\ell(f_\theta)\leq\theta\mathcal{R}_\ell(f^{\gamma_1})+(1-\theta)\mathcal{R}_\ell(f^{\gamma_2})=\theta\nu(\gamma_1)+(1-\theta)\nu(\gamma_2).$$ Therefore, $\nu$ is convex.
\end{proof} 

\textbf{Proof of \cref{thm:estmationerror}}:
\begin{proof}
For any $0\leq\gamma-\varepsilon_0<\gamma-\varepsilon<1$, based on the properties of $\nu(\cdot)$ we have

$$\begin{aligned}
  \frac{\nu(\gamma-\varepsilon_0)-\nu(\gamma-\varepsilon)}{\varepsilon-\varepsilon_0}&\leq\frac{\nu(\gamma-\varepsilon)-\nu(\gamma)}{-\varepsilon}\\
  \nu(\gamma-\varepsilon)-\nu(\gamma)&\leq\frac{\varepsilon}{\varepsilon_0-\varepsilon}\left(\nu(\gamma-\varepsilon_0)-\nu(\gamma-\varepsilon)\right).\\
\end{aligned}$$

Now take $\varepsilon_0=\gamma$, we obtain \begin{equation}\label{eq:propInduced}
    \nu(\gamma-\varepsilon)-\nu(\gamma)\leq\frac{\varepsilon}{\gamma-\varepsilon}\left(2+\frac{\delta}{2}\right)
\end{equation} because we have $f\equiv1+\frac{\delta}{2}$ satisfy the $\ell$-type I error and then $\mathcal{R}_\ell\equiv2+\frac{\delta}{2}$.

Let's first define $A=\left\{\mathbb{E}_Q\left[\ell(Yf(\bm X))\mid Y=1\right]-\frac{1}{n_1}\sum_{i=1}^{n_1}\ell(f(\bm x_i))<\varepsilon\right\}$, where $$\varepsilon=\frac{(\sqrt{2}s+2)c\kappa\left(2+3\sqrt{2\log\frac{2}{\zeta}}\right)}{\sqrt{n_1}}.$$ Based on the proof for \cref{thm:type1bnd}, we have $\mathbb{P}[A]\geq 1-\zeta$.

\begin{equation}\label{eq:profexcess}
\begin{aligned}
\mathcal{R}_\ell(\hat f)-\inf\limits_{f\in\mathcal{F}^+_{s, s^\prime}(\gamma)} R_\ell(f)&=\mathcal{R}_\ell(\hat f)-\inf\limits_{f\in\mathcal{\widehat{F}}^+_{s,s^\prime}(\gamma-\varepsilon)}\mathcal{R}_\ell(f)\\
&\quad +\inf\limits_{f\in\mathcal{\widehat{F}}^+_{s,s^\prime}(\gamma-\varepsilon)}\mathcal{R}_\ell(f)-\inf\limits_{f\in\mathcal{F}^+_{s, s^\prime}(\gamma-2\varepsilon)}\mathcal{R}_\ell(f)\\
&\quad +\inf\limits_{f\in\mathcal{F}^+_{s, s^\prime}(\gamma-2\varepsilon)}\mathcal{R}_\ell(f)-\inf\limits_{f\in\mathcal{F}^+_{s,s^\prime}(\gamma)} R_\ell(f)\\
&\leq2\sup\limits_{f\in\mathcal{\widehat{F}}^+_{s,s^\prime}(\gamma)}\biggl|\mathcal{R}_\ell(f)-\frac{1}{m}\sum\limits_{j=1}^m\ell(-f(\bm x_j))\biggl|\\
&\quad +0\\
&\leq\frac{2\varepsilon}{\gamma-2\varepsilon}\left(2+\frac{\delta}{2}\right) ~~~\mbox{by Inequality}~(\ref{eq:propInduced})
\end{aligned}
\end{equation}

Recall empirical minimizer $\hat f\in\mathcal{\widehat{F}}^+_{s,s^\prime}(\gamma-\varepsilon)\subset \mathcal{F}^+_{s,s^\prime}(\gamma)$.  Define $\bar f:=\arginf_{f\in\mathcal{\widehat F}^+_{s, s^\prime}(\gamma-\varepsilon)} R_\ell(f)$, then the first part in first line on the right of the Inequality (\ref{eq:profexcess}) bounded by twice of supremum is due to
$$\begin{aligned}
\mathcal{R}_\ell(\hat f)-\mathcal{R}_\ell(\bar f)&=\mathcal{R}_\ell(\hat f)-\frac{1}{m}\sum\limits_{j=1}^m\ell(-\hat f(\bm x_j))+\frac{1}{m}\sum\limits_{j=1}^m\ell(-\hat f(\bm x_j))-\frac{1}{m}\sum\limits_{j=1}^m\ell(-\bar f(\bm x_j))\\
&\quad +\frac{1}{m}\sum\limits_{j=1}^m\ell(-\bar f(\bm x_j))-\mathcal{R}_\ell(\bar f)\\
&\leq \mathcal{R}_\ell(\hat f)-\frac{1}{m}\sum\limits_{j=1}^m\ell(-\hat f(\bm x_j))+0- \left[\mathcal{R}_\ell(\bar f)-\frac{1}{m}\sum\limits_{j=1}^m\ell(-\bar f(\bm x_j))\right]\\
&\leq 2\sup\limits_{f\in\mathcal{\widehat{F}}^+_{s,s^\prime}(\gamma-\varepsilon)}\biggl|\mathcal{R}_\ell(f)-\frac{1}{m}\sum\limits_{j=1}^m\ell(-f(\bm x_j))\biggl|\\
&\leq 2\sup\limits_{f\in\mathcal{\widehat{F}}^+_{s,s^\prime}(\gamma)}\biggl|\mathcal{R}_\ell(f)-\frac{1}{m}\sum\limits_{j=1}^m\ell(-f(\bm x_j))\biggl|.
\end{aligned}$$
The second line on the right of the Inequality (\ref{eq:profexcess}) can be bounded by 0 since $\mathcal{F}^+_{s, s^\prime}(\gamma-2\varepsilon)$ is a subspace of $\mathcal{F}^+_{s, s^\prime}(\gamma-\varepsilon)$ and close to $\mathcal{\widehat{F}}^+_{s,s^\prime}(\gamma-2\varepsilon)$ when sample size $m$ is large enough.

Therefore, with probability $1-2\zeta$, $$\mathcal{R}_\ell(\hat f)-\inf\limits_{f\in\mathcal{{F}}^+_{s,s^\prime}(\gamma)} R_\ell(f)\leq\frac{2(\sqrt{2}s+2)c\kappa(2+3\sqrt{2\log\frac{2}{\zeta}})}{\sqrt{m}}+\frac{(4+\delta)\varepsilon}{\gamma-2\varepsilon}.$$
\end{proof}


\textbf{Proof of \cref{thm:featsel}}:
We mainly follow the proof of the Theorem in \citet{chen2018double}. Suppose $\|\bm x\|_\infty = \kappa_0<\infty$, loss function $\ell$ is differentiable with Lipschitz constant $c$. 

\begin{proof}
Let $\hat f$ be the empirical risk minimizer. Based on Corollary 4.36 (RKHSs of differentiable kernels) in \citet{steinwart2008support} we have $\frac{\partial f(\bm x)}{\partial d_t}\leq\frac{\sqrt{2}\kappa_0}{\sigma}$ and hence $\frac{\partial \ell(f)}{\partial d_t}$ is still Lipschitz with Lipschitz constant $c^\prime=\frac{\sqrt{2}\kappa_0 c}{\sigma}$. Then similarly to proofs of previous Theorem 1 and 2, with probability at least $1-3\zeta$, we have $$\left|\frac{\partial}{\partial d_t}\bigg\{\mathbb{E}_{\mathcal{Q}}[\ell(-\hat f(\bm X))]-\frac{1}{m}\sum\limits_{j=1}^m\ell(-\hat f(\bm x_j))\bigg\}\right|\leq \frac{ (\sqrt{2}s+2)c^\prime\kappa(2+3\sqrt{2\log\frac{2}{\zeta}})}{\sqrt{m}},$$
and $$\left|\frac{\partial}{\partial d_t}\bigg\{\mathbb{E}_{\mathcal{Q}}[\ell(-f^*(\bm X))]- \mathbb{E}_{\mathcal{Q}}[\ell(-\hat f(\bm X))]\bigg\}\right|\leq \frac{2 (\sqrt{2}s+2)c^\prime\kappa(2+3\sqrt{2\log\frac{2}{\zeta}})}{\sqrt{m}}+\frac{(4+\delta)\varepsilon}{\gamma-2\varepsilon}+D_s,$$
where the approximation error $D_s:=\inf_{f\in\mathcal{F}}\mathbb{E}_{\mathcal{Q}}[\ell(-f(\bm X))]-\mathbb{E}_{\mathcal{Q}}[\ell(-f^*(\bm X))]\rightarrow0$ since $f^*$ has a sparse representation. Therefore, by the triangle inequality, with probability at least $1-3\zeta$ we have $$\begin{aligned}
&~ \left|\frac{\partial}{\partial d_t}\bigg\{ \mathbb{E}_{\mathcal{Q}}[\ell(-f^*(\bm X))] - \frac{1}{m}\sum\limits_{j=1}^m\ell(-\hat f(\bm x_j))\bigg\} \right|_{\scriptscriptstyle d_t=0, d_{t^\prime}=d^\ast_{t^\prime}, t\neq t^\prime}
\\
\leq&~ \frac{3(\sqrt{2}s+2)c^\prime\kappa(2+3\sqrt{2\log\frac{2}{\zeta}})}{\sqrt{m}}+\frac{(4+\delta)\varepsilon}{\gamma-2\varepsilon}+D_s,
\end{aligned}$$
Consequently, 
for those important features $\bm x_{\cdot, t}$ we have \begin{equation*}
\begin{aligned}
    &~ \left.\frac{\partial}{\partial d_t}\frac{1}{m}\sum\limits_{j=1}^m\ell(-\hat f(\bm x_j))\bigg\} \right|_{\scriptscriptstyle d_t=0, d_{t^\prime}=d^\ast_{t^\prime}, t\neq t^\prime}\\
    < &~ \left.\frac{\partial \mathbb{E}_{\mathcal{Q}}[\ell(-f^*(\bm X))]}{\partial d_t} \right|_{\scriptscriptstyle d_t=0, d_{t^\prime}=d^\ast_{t^\prime}, t\neq t^\prime}+O\left(\max(\frac{1}{\sqrt{n_1}}+\frac{1}{\sqrt{m}}, D_s)\right),
\end{aligned}
\end{equation*}
and for those noise features $\bm x_{\cdot, t}$ we have 
\begin{equation*}
    \begin{aligned}
        &~\left.\frac{\partial}{\partial d_t}\frac{1}{m}\sum\limits_{j=1}^m\ell(-\hat f(\bm x_j))\bigg\} \right|_{\scriptscriptstyle d_t=0, d_{t^\prime}=d^\ast_{t^\prime}, t\neq t^\prime}\\
        \geq &~ \left.\frac{\partial \mathbb{E}_{\mathcal{Q}}[\ell(-f^*(\bm X))]}{\partial d_t} \right|_{\scriptscriptstyle d_t=0, d_{t^\prime}=d^\ast_{t^\prime}, t\neq t^\prime} -O\left(\max(\frac{1}{\sqrt{n_1}}+\frac{1}{\sqrt{m}}, D_s)\right).
    \end{aligned}
\end{equation*}

Finally, together with the assumption in \cref{thm:featsel} for important and unimportant features, we have 
$$\mathbb{P}\left[\mbox{sign}(\hat d_t)=\mbox{sign}(d^\ast_t)\right]\longrightarrow1, \ t\in[p].$$
\end{proof}

\end{document}